\documentclass{article}
\usepackage{amsthm,amssymb,amsmath,amsfonts}
\usepackage{lmodern}
\usepackage{adjustbox}

\usepackage{algorithm}
\usepackage{algorithmic}

\usepackage{tikz}

\allowdisplaybreaks

\let\originalleft\left
\let\originalright\right
\renewcommand{\left}{\mathopen{}\mathclose\bgroup\originalleft}
\renewcommand{\right}{\aftergroup\egroup\originalright}

\newtheorem{assumption}{Assumption}
 \newtheorem{lemma}{Lemma}
 \newtheorem{theorem}{Theorem}
 \newtheorem*{theorem*}{Theorem}

 \newtheorem{proposition}{Proposition}

\definecolor{pear}{HTML}{c60404}
\definecolor{lightblue}{HTML}{2980B9}

\newcommand{\set}[1]{\left\{#1\right\}}
\newcommand{\pa}[1]{\left(#1\right)}
\newcommand{\abs}[1]{\left|#1\right|}
\newcommand{\norm}[1]{\left\|#1\right\|}
\newcommand{\imp}{\Rightarrow}

\newcommand{\arms}{n}

\newcommand{\sign}{\mathrm{sign}}

\DeclareMathOperator*{\argmax}{arg\,max}

\newcommand{\ceil}[1]{\left\lceil#1\right\rceil}
\newcommand{\floor}[1]{\left\lfloor#1\right\rfloor}

\newcommand{\transpose}{^\mathsf{\scriptscriptstyle T}}

\newcommand{\bmu}{{\boldsymbol \mu}}
\newcommand{\btheta}{{\boldsymbol \theta}}
\newcommand{\bdelta}{{\boldsymbol \delta}}
\newcommand{\blambda}{{\boldsymbol \lambda}}

\newcommand{\bheta}{{\boldsymbol \eta}}
\newcommand{\bkappa}{{\boldsymbol \kappa}}

\newcommand{\bGamma}{{\boldsymbol \Gamma}}

\newcommand{\ba}{{\bf a}}

\newcommand{\bb}{{\bf b}}

\newcommand{\bC}{{\bf C}}
\newcommand{\bD}{{\bf D}}

\newcommand{\be}{{\bf e}}

\newcommand{\bk}{{\bf k}}

\newcommand{\bs}{{\bf s}}

\newcommand{\bu}{{\bf u}}
\newcommand{\bU}{{\bf U}}

\newcommand{\bV}{{\bf V}}

\newcommand{\bx}{{\bf x}}
\newcommand{\bX}{{\bf X}}
\newcommand{\bY}{{\bf Y}}

\newcommand{\R}{\mathbb{R}}
\newcommand{\N}{\mathbb{N}}

\newcommand{\cA}{\mathcal{A}}

\newcommand{\cF}{\mathcal{F}}

\newcommand{\cH}{\mathcal{H}}

\newcommand{\cK}{\mathcal{K}}

\newcommand{\cN}{\mathcal{N}}
\newcommand{\cO}{\mathcal{O}}
\newcommand{\cP}{\mathcal{P}}

\newcommand{\cU}{\mathcal{U}}

\newcommand{\EE}[1]{\mathbb{E}\left[#1\right]}

\newcommand{\EEc}[2]{\mathbb{E}\left[\left.#1\right|#2\right]}

\newcommand{\PP}[1]{\mathbb{P}\left[#1\right]}

\newcommand{\Prb}{\mathbb{P}}

\newcommand{\PPc}[2]{\mathbb{P}\left[\left.#1\right|#2\right]}

\newcommand{\II}[1]{\mathbb{I}{\left\{#1\right\}}}

\newcommand{\Bernoulli}{\mathrm{Bernoulli}}

\newcommand{\mean}[1]{\bar\mu_{#1}}

\newcommand{\fA}{\mathfrak{A}}
\newcommand{\fB}{\mathfrak{B}}
\newcommand{\fC}{\mathfrak{C}}
\newcommand{\fD}{\mathfrak{D}}

\newcommand{\fJ}{\mathfrak{J}}

\newcommand{\fM}{\mathfrak{M}}

\newcommand{\fR}{\mathfrak{R}}
\newcommand{\fS}{\mathfrak{S}}
\newcommand{\fT}{\mathfrak{T}}

\newcommand{\fZ}{\mathfrak{Z}}

\newcommand{\eps}{\varepsilon}
\renewcommand{\tilde}{\widetilde}
\renewcommand{\bar}{\overline}

\newcommand{\counter}[1]{N_{#1}}
\newcommand{\ora}{\mathrm{Oracle}}

\makeatletter
\newcommand{\numrel}[2]{
  \refstepcounter{equation}
  \ltx@label{#2}
  \overset{(\theequation)}{#1}
}
\newcommand{\numterm}[1]{\refstepcounter{equation} \ltx@label{#1} (\theequation)}
\makeatother

\newcounter{mylabelcounter}

\makeatletter
\newcommand{\labelText}[2]{%
#1\refstepcounter{mylabelcounter}%
\immediate\write\@auxout{%
  \string\newlabel{#2}{{1}{\thepage}{{\unexpanded{#1}}}{mylabelcounter.\number\value{mylabelcounter}}{}}%
}%
}
\makeatother
\usepackage[colorinlistoftodos, textwidth=22mm, shadow]{todonotes}

     \usepackage[final]{neurips_2020}

\usepackage{hyperref}
\hypersetup{colorlinks,linkcolor={red},citecolor={blue},urlcolor={brown}} 

\usepackage[utf8]{inputenc} 
\usepackage[T1]{fontenc}    
\usepackage{hyperref}       
\usepackage{url}            
\usepackage{booktabs}       
\usepackage{amsfonts}       
\usepackage{nicefrac}       
\usepackage{microtype}      

\usepackage{changepage}     

\title{Statistical Efficiency of Thompson Sampling for Combinatorial Semi-Bandits}

\author{%
  Pierre Perrault \\
  Inria  Lille --- ENS Paris-Saclay\\
  \texttt{pierre.perrault@inria.fr}\\
   \And
  Etienne Boursier\\
  ENS Paris-Saclay \\
   \texttt{etienne.boursier1@gmail.com} \\
   \And
  Vianney Perchet \\
    ENSAE --- Criteo AI Lab \\
   \texttt{vianney.perchet@normalesup.org} 
      \And
  Michal Valko\\
   DeepMind Paris --- Inria  Lille\\
  \texttt{valkom@deepmind.com} 
}

\begin{document}

\maketitle
\begin{abstract}
We investigate stochastic combinatorial multi-armed bandit with semi-bandit feedback (CMAB). In CMAB, the question of the existence of an efficient policy with an optimal asymptotic regret (up to a factor poly-logarithmic with the action size) is still open for many families of distributions, including mutually independent outcomes, and more generally the multivariate \emph{sub-Gaussian} family.  
We propose to answer the above question for these two families by analyzing variants of the \emph{Combinatorial Thompson Sampling} policy (\textsc{cts}). For mutually independent outcomes in $[0,1]$, we propose a tight analysis of \textsc{cts} using Beta priors.
We then look at the more general setting of multivariate sub-Gaussian outcomes and propose a tight analysis of \textsc{cts} using Gaussian priors. This last result gives us an alternative to the \emph{Efficient Sampling for Combinatorial Bandit} policy (\textsc{escb}), which, although optimal, is not computationally efficient.
\end{abstract}
\section{Introduction}
Stochastic multi-armed bandits (MAB) \cite{robbins1952some,berry1985bandit,lai1985asymptotically} are decision-making frameworks in which a learning \emph{agent} acts sequentially in an uncertain environment. At every round $t\in \N^*$, the agent must select one arm from a pool of $n$ arms, denoted by $[n]\triangleq \set{1,\dots,n}$, using a learning \emph{policy} based on the feedback collected
from the previous rounds. Then it obtains as feedback a reward (also called \emph{outcome}) $X_{i,t}\in \R$ --- a random variable sampled from $\Prb_{X_i}$, independently from previous rounds --- where $i$ is the selected arm and  $\Prb_{X_i}$ is a probability distribution --- unknown to the agent --- of mean $\mu_i^*$. The  goal for the agent  is  to maximize the cumulative  reward  over  a  total  of $T$ rounds  ($T$ may be  unknown\footnote{We recall here the fact that in MAB, whether  the horizon $T$ is known or not is not really relevant as algorithms can   be easily adapted \citep{Anytime}.}). The  performance  metric of a policy  is  the  regret,  i.e.,  the expectation of the  difference over $T$ rounds of the cumulative reward between the policy that always picked the arm with the highest expected reward and the learning policy.  MAB models the classical dilemma between exploration and exploitation, i.e.,  whether to continue exploring arms to obtain more information (and thus strengthen the confidence in the estimates of the distributions $\Prb_{X_i}$), or to use the information gathered by playing the best arm according to the observations so far.

In this paper, we study stochastic combinatorial multi-armed bandit (CMAB) \citep{cesa-bianchi2012combinatorial}, which is an extension of MAB where the agent selects a \emph{super arm} (or \emph{action}) $A_t\in \cA\subset \cP([n])$ at each round $t$. The set $\cA$ is the \emph{action space}, defined as a collection of subsets of the (base) arms. 
The kind of reward and feedback varies depending on the problem at hand. We consider the \emph{semi-bandit} setting, where the feedback includes the outcomes of all base arms in the played super arm. Formally, 
the agent observes\footnote{Henceforth, we typeset vectors in bold and indicate components with indices, i.e., $\ba=(a_i)_{i\in [n]} \in \R^n$. We also let $\be_i$ be the $i^{th}$ canonical unit
vector of $\R^n$, and define the incidence vector of any subset $A\subset [n]$ as \(\be_A\triangleq \sum_{i\in A}\be_i.\) We denote by $\ba\odot\bb\triangleq (a_ib_i)$ the Hadamard product of two vectors $\ba$ and $\bb$.
} $\bX_t\odot \be_{A_t}\triangleq \pa{X_{i,t}\II{i\in A_t}}_{i\in [n]}$ and the reward, given the choice of $A_t$, is a function of $\bmu^*\odot \be_{A_t}$ (traditionally, the reward is linear and equal to $\be_{A_t}\transpose \bmu^* $, but our analysis goes beyond this setting).
In recent years, CMAB
has attracted a lot of interest (see e.g. \citet{gai2012combinatorial,chen13a,Chen2015combinatorial,kveton2015tight,wang2017improving,pmlr-v89-perrault19a,perraultbudgeted2020}), particularly due to its wide applications in network routing, online advertising, recommender system, influence marketing, etc.

In CMAB, the whole joint distribution of the vector of outcomes $\bX$ matters, contrary to standard MAB where only the  marginals are sufficient to characterize a problem instance. For example,
 the following two extreme problem instances are distinct within the CMAB framework:
 \begin{itemize}
  \item[\color{red}\labelText{$(i)$}{indep_setting}] Each $\Prb_{X_i}$ is sub-Gaussian and the arm distributions are mutually independent, i.e., $\Prb_{\bX}=\otimes_{i\in [n]}\Prb_{X_i}$.
     \item[\color{red}\labelText{$(ii)$}{arb_cor_setting}]  Each $\Prb_{X_i}$ is sub-Gaussian but the stochastic dependencies between the arm distributions are ``worst case'': the performance metric is the supremum of the regret over all possible dependencies between the marginals.
 \end{itemize}
 Those two settings are indeed different as two different lower bounds on the asymptotic (in $T$) regret can be derived. In particular, the regret scales as $\Omega\pa{   n\log(T)/\Delta}$ for the setting \nameref{indep_setting}, and as $\Omega\pa{ m  n\log(T)/\Delta}$ for \nameref{arb_cor_setting},
 where $\Delta$ is the minimum gap in the expected reward between an optimal super arm and any non-optimal super arm, and where
 $m\triangleq\max_{A\in \cA}\abs{A}$.

 Many CMAB policies are based on the \emph{Upper Confidence Bound} (UCB) approach, extending the classical \textsc{ucb} policy  \citep{auer2002finite} from MAB to CMAB. This type of approach uses an optimistic estimate $\bmu_t$ of $\bmu^*$ (i.e., for which the reward function is overestimated), lying in a well-chosen confidence region.
 For setting \nameref{arb_cor_setting}, there exist UCB-style policies that match the lower bound mentioned above. 
 An example of such policy is \emph{Combinatorial Upper Confidence Bound} (\textsc{cucb}) \citep{chen13a,kveton2015tight}, that uses a Cartesian product of the individual confidence intervals of each arm as a confidence region. For setting \nameref{indep_setting},
 \citet{combes2015combinatorial} provided the UCB-style policy  \emph{Efficient Sampling for Combinatorial Bandit} (\textsc{escb}), that uses the assumption of mutual independence between arm distributions in order to build a tighter ellipsoidal confidence region around the empirical mean, which helps to better restrict the exploration. \citet{Degenne2016} gave the following generalization of setting \nameref{indep_setting}:
 \begin{itemize}
     \item[\color{red}\labelText{$(iii)$}{subgau_setting}] The joint probability $\Prb_{\bX}$ is $\bC$-sub-Gaussian, for a positive semi-definite matrix $\bC\succeq 0$, i.e.,   $\EE{e^{\blambda\transpose \pa{\bX-\bmu^*}}}\leq e^{\blambda\transpose\bC\blambda/2}$, for all $\blambda\in \R^n$.
 \end{itemize}
 In this case,  they  provided the policy \textsc{ols-ucb}, leveraging this additional assumption and such that it essentially reduces to \textsc{escb} in the specific case of  diagonal matrix $\bC$ with a regret bound of $\cO\pa{ \log^2(m)  n\log(T)/\Delta}$) (so it matches the above lower bound up to a polylogarithmic factor in $m$). We refer the reader to Table~\ref{table:regretofallalgos} for an overview of the above regret (lower) bounds.
 
 \begin{table}
  \caption{Factor in front of $n\log(T)/\Delta$ in the regret bound ($\cO\pa{\cdot}$ for upper bounds), computationally inefficient policies are printed with a subscript $*$, setting \nameref{subgau_setting} is for $\bC$ diagonal, \textsc{clip cts-gaussian} is for linear reward functions, and with only $\blambda\in\R_+^n$ in \nameref{subgau_setting}. Our results are printed in bold, see Theorem~\ref{thm:tsbeta}, Theorem~\ref{thm:tsgauss}, Theorem~\ref{thm:tsclipgauss} related to \textsc{cts-beta}, \textsc{cts-gaussian}, \textsc{clip cts-gaussian} respectively.}
  \label{sample-table}\begin{adjustwidth}{-0.5cm}{-0.5cm}\centering
  \begin{tabular}{c|cccccc}
    \toprule
         & \textsc{cucb}     & $\textsc{escb}_*$ &\textsc{cts-beta}& \textsc{cts-gaussian} & \textsc{clip cts-gaussian} & Lower bound  \\
    \midrule
    \nameref{indep_setting}& ${  m}$  & ${\log^2(m)}$ &\mathversion{bold}${\log^2(m)}$ &\mathversion{bold}${\log^2(m)}$  &\mathversion{bold}${\log^2(m)}$  & $\Omega\pa{1  }$  \\
    \nameref{arb_cor_setting}     & ${  m}$& ${  m}$ & - & \mathversion{bold}$   {{\log^2(m)  m}}$& \mathversion{bold}$m$ & $\Omega\pa{  m}$   \\
   \nameref{subgau_setting}    & ${  m}$& ${\log^2(m)}$ & - & \mathversion{bold}$   {{\log^2(m)  }}$& \mathversion{bold}${{\log^2(m)  }}$ & $\Omega\pa{1}$   \\
    \bottomrule
  \end{tabular}\label{table:regretofallalgos}\end{adjustwidth}\vspace{-5pt}
\end{table}
%
%
 
 %
   In some CMAB problems, the action space $\cA$ and the reward function are simple enough for  the existence of an exact \emph{oracle} that  takes as input a vector $\bmu\in \R^n$ and  outputs the solution of the combinatorial problem (associated to the mean vector $\bmu$),
 with a polynomial time complexity $\cO\pa{\text{poly}(n)}$. 
 Under this assumption (referred to as Assumption~\ref{ass:efficient}), \textsc{cucb}, that plays the action $A_t=\ora\pa{\bmu_t}$ at round $t$, is efficient to implement, and has a $\cO\pa{\text{poly}(n)}$ time complexity per round. In that case, the setting \nameref{arb_cor_setting} is therefore essentially solved. 
On the other hand, this is not true for the settings \nameref{indep_setting} and \nameref{subgau_setting}, as
\textsc{escb} needs to solve a difficult combinatorial problem in each round (NP-Hard in general \citep{atamturk2017maximizing}).

The inefficiency of \textsc{escb} triggered some attempts to implement an efficient version:
 \citet{perrault2019exploiting} proposed an efficient approximation method for implementing \textsc{escb} in the case the action space has a \emph{matroid} structure: they prove a time complexity of $\cO\pa{\text{poly}(n)}$ while keeping the same regret rate. However,
this improvement is mitigated by the fact that \textsc{cucb} reaches the optimal regret rate $\cO\pa{   n\log(T)/\Delta}$ for the special case of matroid  semi-bandits  \citep{anantharam1987asymptotically,kveton2014matroid,Talebi2016}. Recently,
\citet{cuvelier2020statistically} provided another approach for approximating \textsc{escb} for a wide variety of action spaces, including the  matching  bandit setting \citep{gai2010learning} and the  online  shortest  path  problem \citep{liu2012adaptive}, where \textsc{cucb} is not known to be better than \textsc{escb}. However, their policies are still computationally expensive when $T$ is large, since the time complexity at round $t$ is of order $\cO\pa{t\cdot\text{poly}(n)}$. 

Another line of research is to find an efficient alternative to \textsc{escb}.
One of the most promising candidate is \emph{Thompson Sampling} (\textsc{ts}). Although introduced much earlier by \citet{thompson1933likelihood}, the theoretical analysis of \textsc{ts} for frequentist MAB is quite recent: \citet{kaufmann2012thompson,agrawal2012thompsonarxiv} gave
a regret bound matching the \textsc{ucb} policy theoretically.  Moreover, \textsc{ts} often performs better than \textsc{ucb} in practice, making \textsc{ts} an attractive policy for further investigations. For CMAB, \textsc{ts} extends to \emph{Combinatorial Thompson Sampling}
(\textsc{cts}). In \textsc{cts}, the unknown mean $\bmu^*$ is associated with a belief (a prior distribution)  updated to a posterior with the Bayes’rule, each time a feedback is received. In order to choose an action at round $t$, \textsc{cts} draws a sample $\btheta_t$ from the current belief, and plays the action given by $\ora\pa{\btheta_t}$.
\textsc{cts}
is attractive also because its time complexity is $\cO\pa{\text{poly}(n)}$ under Assumption~\ref{ass:efficient}. Recently, for the setting \nameref{indep_setting} with bounded outcomes, \citet{Wang2018} proposed an analysis of \textsc{cts-beta}, which is \textsc{cts} where the prior distribution is chosen to be a product of $n$ Beta distributions. They proved two regret upper bounds depending on the class of reward functions: 
\begin{align}\cO\pa{\frac{n \sqrt{m}\log(T)}{\Delta}} \text{ in the linear case and }\cO\pa{\frac{n {m}\log(T)}{\Delta}} \label{rel:ratewang}~\text{in the general case. }\end{align}
Although the aforementioned upper bound in the linear reward case outperforms the one of \textsc{cucb}, it doesn't match the one of \textsc{escb}. To summarize,  and despite many efforts, the existence of a policy that is both optimal (up to a polylogarithmic factor in $m$) and efficient in the setting \nameref{indep_setting} or \nameref{subgau_setting} is still an open problem, which we tackle in this paper.
 \vspace{-.1cm}
\paragraph{Further related work} We refer the reader to \citet{Wang2018} for further related work on \textsc{ts} for combinatorial bandits, and particularly for \citet{gopalan2013thompson}, that provided a frequentist high-probability regret bounds for \textsc{ts} with
 a general action space and a general feedback model ---  
\citet{Komiyama2015}, that investigated \textsc{ts} for the $m$-sets action space --- 
\citet{wen2015efficient}, that studied \textsc{ts} for
contextual CMAB problems, using the Bayesian regret metric (see also  \citet{russo2016information}).
 \vspace{-.15cm}
\subsection{Contributions}
 \vspace{-.1cm}
 We first improve the result of \citet{Wang2018} by providing the regret upper bound $\cO\pa{ \log^2(m)  n\log(T)/\Delta}$ for \textsc{cts-beta} in the setting \nameref{indep_setting} with bounded outcomes. This bound is valid even for non linear reward functions. 
Our main contribution is a regret bound for the setting~\nameref{subgau_setting}.
We propose an efficient policy called \textsc{cts-gaussian}, that is \textsc{cts}
where the prior distribution is chosen to be a
multivariate Gaussian. An analysis of \textsc{cts-gaussian} allows us to obtain a regret bound reducing to $\cO\pa{ \log^2(m)  n\log(T)/\Delta}$ for a diagonal sub-Gaussian matrix. When the reward function is linear,
we generalize the setting \nameref{subgau_setting} assuming only $\blambda\in \R_+^n$. This allows us to get rid of negative correlations between the outcomes (as in \cite{perrault2020covariance-adapting}), and focus on positive correlations. We propose in this setting the policy
 \textsc{clip cts-gaussian}, where the score is truncated from below with the empirical mean, and from above with the UCB. Truncations from above are not necessary, but can limit optimism, especially when positive correlations are significant. We obtain an improved regret bound for \textsc{clip cts-gaussian}, where negative correlations no longer appear
in the regret bound and where, in setting~\nameref{arb_cor_setting}, the extra $\log^2(m)$ factor present in the regret bound of \textsc{cts-gaussian} disappears. All these results are summarized and compared to other state-of-the-art policies in Table~\ref{table:regretofallalgos}.

\section{Model}\label{sec:model}
 CMAB is formally introduced as follows.
Consider a random process $\pa{\bX_t}\overset{iid}{\sim}\Prb_{\bX}$, where $\Prb_{\bX}$ is a distribution --- unknown to the agent --- of random vectors in $\R^n$, with unknown mean~$\bmu^*$. 
At each round $t\in [T]$, the agent chooses a super arm (or action) $A_t\in \cA\subset \cP([n])$ based on the history of observations $\cH_t\triangleq \sigma\pa{\bX_1\odot \be_{A_1},\dots,\bX_{t-1}\odot \be_{A_{t-1}}}$ and a possible extra source of randomness (we denote by $\cF_t$ the filtration containing $\cH_t$ and the extra randomness of round $t$ --- in particular, $A_t\in \cF_t$). The feedback received is then $\bX_t\odot \be_{A_t}$ and the associated expected reward of the agent at that stage is $r(A_t,\bmu^*)$, for some known function $r$. The objective of the agent is to
 minimize the regret, defined for a policy $\pi$ as 
\[\forall T\geq 1,\quad R_T\pa{\pi}\triangleq\EE{\sum_{t=1}^T \Delta_t},\]
where $\Delta_t\triangleq \Delta\pa{A_t} \triangleq r(A^*,\bmu^*)-r(A_t,\bmu^*)$ with  $A^*\in \argmax_{A'\in \cA}r(A',\bmu^*)$.
As stated in the introduction, we will assume the following:
\begin{assumption}\label{ass:efficient}
The agent has access to an oracle with a time complexity $\cO\pa{\mathrm{poly}\pa{n}}$ such that for any mean vector $\bmu$, $\ora\pa{\bmu}\in \argmax_{A\in \cA} r(A,\bmu)$.
\end{assumption}
Similar to \citet{Chen2015combinatorial}, we assume that the function $r$ satisfies the following smoothness property.
\begin{assumption}\label{ass:smooth}
There exists a constant $B$, such that for every super arm $A\in \cA$ and every pair of mean vectors $\bmu$ and $\bmu'$,
\(\abs{r(A,\bmu)-r(A,\bmu')}\leq B\norm{\be_{A}\odot\pa{\bmu-\bmu'}}_1.\)
\end{assumption}
For an arm $i\in [n]$, we define the number of time $i$ has been chosen at the beginning of round $t$ as 
 \(N_{i,t-1}\triangleq\sum_{t'\in[t-1]}\II{i\in A_{t'}}.\)
 We also define the following quantities, that will be useful in the expression of an upper bound on the regret:

\hspace{1cm}  $m^*\triangleq\min_{A\in \argmax_{A'\in \cA}\be_{A'}\transpose\bmu^*}\abs{A}$ is the minimum size of an optimal action,\vspace{-0.25cm}

\hspace{1cm}  \(\Delta_{i,\min}\triangleq \min_{A\in \cA,~\Delta\pa{A}>0,~i\in A}\Delta\pa{A},\) is the minimal gap of an action  containing $i \in [n]$,\vspace{-0.25cm}

\hspace{1cm}  \(\Delta_{\min}\triangleq \min_{i\in [n]}\Delta_{i,\min},\) is the minimal arm-gap and \vspace{-0.25cm}

\hspace{1cm}  \(\Delta_{\max}\triangleq \max_{A\in \cA}\Delta\pa{A}\) is the maximal gap.
  
 \section{Regret bound for \textsc{cts-beta} in setting \nameref{indep_setting}}\label{sec:beta}
  \begin{algorithm}[t]
\begin{algorithmic}
\STATE \textbf{Initialization}: 
 For each arm $i$, let $a_i=b_i=1$.
\STATE \textbf{For all} $t\geq 1$:
\STATE \quad Draw $\btheta_t\sim\otimes_{i\in [n]}\text{Beta}(a_i,b_i)$, and play $A_t=\ora\pa{\btheta_t}$.
\STATE \quad Get the observation $\bX_{t}\odot\be_{A_t}$, and draw $\bY_t\sim \otimes_{i\in A_t}\Bernoulli(X_{i,t})$.
\STATE \quad For all $i\in A_t$ update $a_i\leftarrow a_i+Y_{i,t}$ and $b_i\leftarrow b_i+1-Y_{i,t}$. 
\end{algorithmic}
\caption{\textsc{cts-beta}}\label{algo:tsbeta}
\end{algorithm}
 
 In this section, we consider the following assumption on top of the CMAB setting from section~\ref{sec:model}.
 \begin{assumption}
 The outcomes $X_i$ are bounded (in $[0,1]$, w.l.o.g.), and are mutually independent (we are thus in a special case of \nameref{indep_setting}).
 \end{assumption}
 For this problem, we consider \textsc{cts-beta} in Algorithm~\ref{algo:tsbeta}, which is described as follows. The prior is set to be a product of $n$ beta distributions (being thus uniform over $[0,1]$ initially). Notice, this prior is conjugate to a product of Bernoulli distributions. After the agent get an observation $X_{i,t}$, it first binarizes it by sampling $Y_{i,t}\sim\Bernoulli\pa{X_{i,t}}$ (the regret of the problem defined by the observations $Y_{i,t}$ is the same because  $\EE{Y_{i,t}}=\mu_i
^*$). Then the prior is updated using Bayes’ rule with each sample $Y_{i,t}$.
When choosing a super arm at round $t$, the agent draws $\btheta_{t}$ from the beta belief, and then plugged it into the oracle, which outputs the super arm $A_t$ to play.

The main result of this section is
Theorem~\ref{thm:tsbeta}, that  improves the regret bound of \citet{Wang2018} for \textsc{cts-beta}.
 \begin{theorem}\label{thm:tsbeta}The policy $\pi$ described in
 Algorithm~\ref{algo:tsbeta} has regret $R_T(\pi)$ of order \[\cO\pa{\sum_{i\in [n]} \frac{B^2\log^2(m)\log(T)}{\Delta_{i,\min}}}\cdot\]
 \end{theorem}
 The proof of Theorem~\ref{thm:tsbeta}, as well as the complete non-asymptotic upper-bound is postponed to Appendix~\ref{app:thm:tsbeta}.
 Our analysis incorporates two novelties that we detail in the two following paragraphs.

 \paragraph{An improved leading term}(cf.\,Step 3 of the proof of Theorem~\ref{thm:tsbeta} in Appendix~\ref{app:thm:tsbeta})
 We define the \emph{empirical average} of each arm $i\in [n]$ at the beginning of round $t$ as
\(\mean{i,t-1}\triangleq\sum_{t'\in[t-1]}\frac{\II{i\in A_{t'}}Y_{i,t'}}{\counter{i,t-1}}.\)
Notice that this empirical average definition differs from the one that is classically used in CMAB, since samples $Y_{i,t'}$ are used rather than $X_{i,t'}$. The improved dependence in $m$ in the leading term of Theorem~\ref{thm:tsbeta} (compared to \eqref{rel:ratewang}) is a consequence of two ingredients. The first is the following concentration inequality (see Appendix~\ref{app:thm:tsbeta}, Lemma~\ref{lem:nlem4}), which improves that of \citet{Wang2018} by extending it to the case of non-linear reward. Indeed, we rather control the $\ell_1$ norm in this case, instead of the $\ell_\infty$-norm, which leads to a tighter bound.
 \begin{align} \PPc{\norm{\be_{A_t}\odot\pa{\btheta_{t}-\bar\bmu_{t-1}} }_1\geq\sqrt{\frac{1}{2}{\log\pa{\abs{\cA}2^mT}}\sum_{i\in {A_t}}\frac{1}{N_{i,t-1}}}}{\cH_t}\leq 1/T.\label{rel:concentration_l1_ts}\end{align}
The second ingredient is a more careful handling of the square-root term in the above probability, based on a method similar to the one in
 \citet{Degenne2016}.
 
\paragraph{$T$-independent term} (cf.\,Step 4 of the proof of Theorem~\ref{thm:tsbeta} in Appendix~\ref{app:thm:tsbeta})
 Similarly to \citet{Wang2018}, our regret bound also contains an exponential term that is constant in $T$. Note, however that the term of \citet{Wang2018} is of order $\cO(\eps^{-2m^*-2})$, whereas ours is of order $\cO\pa{\eps^{-4m^*-2}}$, where $\eps\in (0,1)$ is of order $\Delta_{\min}/{(m^*)}^2$.  This discrepancy is due to the correction of a minor negligence inaccuracy in their Lemma~7, where they assume, at the end of the proof, that one could decorrelate the counters from the outcomes received. We manage to circumvent this issue by doing a careful union bound over the counters. It is this union bound that brings a larger dependence in this constant term. An additional discussion is deferred to the end of Appendix~\ref{app:thm:tsbeta}.

 \section{Regret bound for \textsc{cts-gausian} in setting \nameref{subgau_setting}}\label{sec:gauss}
 \begin{algorithm}[t]
\begin{algorithmic}
\STATE \textbf{Input}: The vector $\bD$, and a parameter $\beta>1$.
\STATE \textbf{Initialization}: Play each arm once (if the agent knows that $\bmu^*\in [a,b]^n$, this might be skipped)
\STATE \textbf{For every subsequent round} $t$:
\STATE \quad Draw $\btheta_t\sim\otimes_{i\in [n]}\cN\pa{\bar\mu_{i,t-1},{ N^{-1}_{i,t-1}} \beta D_i}$ ($\theta_{i,t}\sim\cU[a,b]$ if $N_{i,t-1}=0$). \STATE \quad Play $A_t=\ora\pa{\btheta_t}$.
\STATE\quad Get the observation $\bX_{t}\odot\be_{A_t}$, and  update $\bar\bmu_{t-1}$ and counters accordingly. 
\end{algorithmic}
\caption{\textsc{cts-gaussian}}\label{algo:tsgauss}
\end{algorithm}
 
 In this section, we consider the setting from section~\ref{sec:model}, with a more general sub-Gaussian family for  $\bX\in \R^n$.  More precisely, we make the following similar  assumption as in
 \citet{Degenne2016}. Proposition~\ref{prop:subgau} gives two examples 
 included in this assumption (see  Appendix~\ref{app:subgau} for a proof).
 \begin{assumption}\label{ass:subgau}
 There exists a vector $\bD\triangleq(D_1,\dots,D_n) \in \R_+^n$ known to the agent such that
\begin{align*}\forall A\in\cA,~\forall \blambda\in \R^\arms~s.t.~\blambda=\blambda\odot\be_A,\quad\EE{e^{\blambda\transpose\pa{\bX-\bmu^*}}}\leq e^{\blambda\transpose\bD\odot\blambda/2}.\end{align*}\end{assumption}

\paragraph{Motivation for sub-Gaussian outcomes}
In the same way as boundedness generalizes to sub-Gaussianity in $1$d, we have that if $\bX$ is a.s. in a compact $\cK$, it is $\bC$-sub-Gaussian, with $\bC$ built from the John's ellipsoid of $\cK$.
In this case, $D_i$ is computed with a linear maximization over $\cA$. In particular, $\cK=B_{\ell_\infty}(0,1)$ gives $D_i=m$, and $\cK=B_{\ell_2}(0,1)$ gives $D_i=1$. We can also use other structures on the outcomes to have $D_i$, such as negative dependence (as we will see in our shortest path experiments, in section~\ref{sec:exp}).

\begin{proposition}\label{prop:subgau}
Assumption~\ref{ass:subgau} encompasses the  $\kappa_i^2$-sub Gaussian outcomes with worst case dependencies between the arm distributions, taking
  $D_i = \kappa_i^2 m$. It also captures $\bC$-sub-Gaussian outcomes with a known sub-Gaussian matrix $\bC$ (setting \nameref{subgau_setting}), taking  $D_i=\max_{A\in \cA,~i\in A}\sum_{j\in A}\abs{C_{ij}}$.
\end{proposition}

 For the above setting, we provide \textsc{cts-gaussian} in Algorithm~\ref{algo:tsgauss},
where we define the empirical mean of arm $i$ at round $t\geq 1$ as
\(\mean{i,t-1}\triangleq\sum_{t'\in[t-1]}\frac{\II{i\in A_{t'}}X_{i,t'}}{\counter{i,t-1}}.\)
This algorithm is comparable to Algorithm~\ref{algo:tsbeta} but considers a Gaussian prior for each arm. Notice, the Gaussian family is \textit{self-conjugate}, so except in the Gaussian-outcomes case, we do not rely on exact conjugated prior here. Although this is not surprising --- since it is known that \textsc{ts} can work without exact conjugate prior with respect to the outcomes --- obtaining an upper bound on the regret of the policy \textsc{cts-gaussian} is non-trivial and constitutes our main contribution. We state our main result in Theorem~\ref{thm:tsgauss}.
\begin{theorem}\label{thm:tsgauss}
  The policy $\pi$ described in
 Algorithm~\ref{algo:tsgauss} has regret $R_T(\pi)$ of order
 \[\cO\pa{\sum_{i\in [n]}\frac{B^2D_i\log^2(m)\log(T)}{\Delta_{i,\min}} }\cdot\]
 \end{theorem}
  The proof of Theorem~\ref{thm:tsgauss}, as well as the complete non-asymptotic upper-bound is postponed to Appendix~\ref{app:tsgauss}. Nonetheless, in the following paragraphs, we provide some insights and highlight the novelty of our analysis.
\paragraph{Main proof challenges}
In the setting of the previous section, the outcomes are independent in $[0,1]$ and an important step in Algorithm~\ref{algo:tsbeta} was to transform the outcomes into binary variables in order to be consistent with the posterior. Here, outcomes are no longer independent. In addition to that, we cannot transform the outcomes into Gaussian variables in the same way as in Algorithm~\ref{algo:tsbeta}. These two points are the main technical challenges to address in our analysis.

\paragraph{Stochastic dominance}
Before providing details on how we deal with the above challenges, first recall that the standard analysis (in the case of a factorized prior, that we have here\footnote{\label{note3}In practice, for $\bC$-sub Gaussian outcomes, the choice $\cN\pa{\bar\bmu_{t-1},\pa{C_{ij}N_{ij,t-1}N^{-1}_{i,t-1}N^{-1}_{j,t-1}}_{ij}}$ for the prior  where 
\(N_{ij,t-1}\triangleq\sum_{t'\in[t-1]}\II{i\in A_{t'}}\II{j\in A_{t'}}\) may be preferred.}) consists in bounding the expected number of rounds needed for the sample $\btheta_t$ to be close to the true mean $\bmu^*$ on a certain set $Z\subset A^*$, i.e., for the event $ \set{\norm{\pa{\bmu^*-\btheta_t}\odot\be_{Z}}_\infty> \eps}$ to happen. We let $\fT_t\pa{Z}$ denote the complementary event. 
As for the proof of Theorem~\ref{thm:tsbeta}, we can condition on the history to rewrite this expected number of rounds and then upper bound it as
\begin{align*}&\EE{\sum_{t\geq 1} (t-1) \PPc{\neg\fT_{t}\pa{Z}}{\cH_{t}} \prod_{j=1}^{t-1}\PPc{\fT_{{j}}\pa{Z}}{\cH_{{j}}}}\\&\leq \EE{\sup_{t\geq 1} \frac{1}{\PPc{\neg\fT_{t}\pa{Z}}{\cH_{t}}}}-1\leq \sum_{Z'\subset Z,~Z'\neq \emptyset} \EE{\sup_{t\geq 1}{\prod_{i\in Z'}\pa{\frac{1}{\PPc{\abs{\theta_{i,t}-\mu_i^*}\leq \eps}{\cH_t}}-1}}}.\end{align*}
Now, using the fact that the conditional distribution of $\theta_{i,t}-\bar\mu_{i,t-1}$ is symmetric and depends only on the counter $N_{i,t-1}$, we obtain that the probability $\PPc{\abs{\theta_{i,t}-\mu_i^*}\leq \eps}{\cH_t}$ is a monotonic function of the deviation $\abs{\bar\mu_{i,t-1}-\mu^*_i}$. Let us emphasize that this property of the Gaussian prior used is crucial and that it is not obvious to transfer the same technique to a beta prior.
 To sum up, we have to control a term of the form $\EE{\sup_{t\geq 1}\prod_{i\in Z'} g_i\pa{\abs{\bar\mu_{i,t-1}-\mu^*_i}}}$, where $g_i$ are non-negative increasing functions.  Our approach is to prove that $\pa{\abs{\bar\mu_{i,t-1}-\mu_i^*}}_i$ is \emph{weakly stochastically dominated} by
 $\pa{\sqrt{\frac{\beta D_i}{N_{i,t-1}}}\abs{\eta_i}}_i$, where $\bheta\sim\otimes_i\cN(0,1)$, which is
 the same vector but where the empirical mean is built with independent Gaussian outcomes instead.  Notice, independence is crucial to be able to factorize the expectation $\EE{\prod_{i\in Z'} g_i}$, in the same way as in the proof of Theorem~\ref{thm:tsbeta}.
 We recall two equivalent definitions of $\bU$ is weakly stochastically dominated by $\bV,$ see \citet{shaked2007stochastic} for more details and properties of dominances,
\begin{itemize}
    \item For all non-negative, non-increasing functions $f_i$, it holds  $\EE{\prod_i f_i(U_i)}\leq \EE{\prod_i f_i(V_i)}$. 
    \item For any vector $\bx$, it holds $\PP{\bU\geq \bx}\leq \PP{\bV\geq \bx}$.
\end{itemize}
The first point applied to $g_i$'s (and up to the supremum over $t$) is a simple way to obtain the aforementioned wanted control. Thus, it's enough to prove the second point, which is a consequence of the sub-Gaussianity of outcomes given by Assumption~\ref{ass:subgau} and 
  some concentration inequality. Finally, we circumvent the supremum over $t\geq 1$ issue thanks to Doob's optional sampling theorem for non-negative super-martingales (see \citet{durrett2019probability}, Theorem 5.7.6).

    \paragraph{Importance of using a factorized prior in our analysis}
    Note that in Algorithm~\ref{algo:tsgauss}, the samples $\theta_{i,t}$ are independent, while the outcomes are not necessarily independent. This independence is in fact crucial in order to be able to start the analysis in the same way as in the proof of Theorem~\ref{thm:tsbeta} (recall that Algorithm~\ref{algo:tsbeta} also uses a factorized prior). More precisely, a factorized prior allows us to link the filtered regret against the event $\fS_t(Z)\wedge\fT_t(Z)$ to the expected number of rounds needed for $\neg\fT_t(Z)$ to occur (see \eqref{rel:llaass} in Step 4 of the proof of Theorem~\ref{thm:tsbeta} in Appendix~\ref{app:thm:tsbeta} for a definition of $\fS_t(Z)$).
Indeed, without the factorized prior, the two events $\fS_t(Z),\fT_t(Z)$ would no longer be independent conditionally to the history, and the term ${1}/{\PPc{\neg\fT_{t}\pa{Z}}{\cH_{t}}}$ obtained in the previous paragraph would then be replaced by ${1}/{\PPc{\neg\fT_{t}\pa{Z}}{\fS_{t}\pa{Z},\cH_{t}}}$, which is much more difficult to deal with. To the best of our knowledge, it is unknown how to get the desired bound when $\fS_t(Z)$ and $\fT_t(Z)$ are not independent conditionally to the history.

\subsection{\textsc{clip cts-gaussian} for the linear reward case}
 In this subsection, we make the following assumptions on top of Section~\ref{sec:model}.
 \begin{assumption}\label{ass:linear_reward}
 The reward function is linear, defined as $r(A,\bmu)\triangleq\be_A\transpose \bmu$.
 \end{assumption}
 \begin{assumption}\label{ass:subgaupos}
 The agent knows a matrix $\bGamma\succeq 0$ s.t.
$\forall\blambda\in \R_+^n,~\EE{e^{\blambda\pa{\bX-\bmu^*}}}\leq e^{\blambda\transpose \bGamma\blambda/2}.$
 \end{assumption}
 Notice that Assumption~\ref{ass:subgaupos} slightly generalises the setting from \citet{Degenne2016}.
 Requiring $\blambda\in \R_+^n$ allows us to take  $D_i=\max_{A\in \cA,~i\in A}\sum_{j\in A}\pa{0\vee \Gamma_{ij}}$, so that negative correlations are no longer harmful. $D_i$ can still be too large (and thus $\btheta_t$ might be over-sampled), so we cap $\btheta_t$ with the score $\bmu_t$ used by \textsc{cucb}.  
 The resulting policy is \textsc{clip cts-gaussian}, where the score $\btheta_t$ is replaced by $\bar\bmu_{t-1}\vee \btheta_t \wedge \bmu_t$ before we plug it into $\ora$, where $\mu_{i,t}=\bar\mu_{i,t-1}+\sqrt{\Gamma_{ii}\frac{2\pa{\log(t)+4\log\log(t)}}{N_{i,t-1}}}$.
 \textsc{clip cts-gaussian} enjoys the following regret bound. 
 \begin{theorem}\label{thm:tsclipgauss}
  The policy \textsc{clip cts-gaussian} has regret of order
   \vspace{-.15cm}
  \[\cO\pa{\sum_{i\in [n]}\frac{\pa{D_i\log^2(m)\wedge m\Gamma_{ii}}\log(T)}{\Delta_{i,\min}}}.\]
 \end{theorem}
 \vspace{-.25cm}
 Not only $D_i$ is improved through the above relaxation, but also, the leading term is never worse than the one of \textsc{cucb}. The proof and the complete non-asymptotic upper-bound is delayed to Appendix~\ref{app:optimistic_cts}. We note that we rely heavily on reward linearity to analyse this clip version, not only using monotony to restrict the controls to the $\R_+^n$ directions (and thus to cap from bellow the sample by the empirical mean), but also using the oracle's invariance property $\ora\pa{\bmu}=\ora\pa{\bmu+\bdelta\odot\be_{\ora\pa{\bmu}}}$, with $\bdelta\geq 0$, to cap the sample from above by the UCB. 
 
  \vspace{-.15cm}
 \paragraph{Comparison with the \textsc{ols-ucb} analysis of \citet{Degenne2016}}
 The leading term in the regret bound given from Theorem~\ref{thm:tsclipgauss} is comparable to the one for \textsc{ols-ucb} from \citet{Degenne2016}. Indeed, we recall that they obtained a factor of order $\Gamma_{ii}\pa{(1-\gamma)\log^2(m) + \gamma m}$, with $\gamma\triangleq\max_{A\in \cA}\max_{(i,j)\in A^2,i\neq j}\pa{0\vee\Gamma_{ij}}/\sqrt{\Gamma_{ii}\Gamma_{jj}}$,  where we have $\pa{D_i\log^2(m)\wedge m\Gamma_{ii} }$. When $\gamma\in\set{0,1}$ (this is the case when we are in the settings \nameref{indep_setting} and \nameref{arb_cor_setting} respectively), these two terms coincide. When $\gamma\in (0,1)$, they are incomparable in general. We can still see that our variance term $D_i$ is always lower than  their $\Gamma_{ii}\pa{(1-\gamma) + \gamma m}$, i.e., that our bound rate is lower than $\log^2(m)$ times theirs. 
  \vspace{-.1cm}
 \section{Experiments}\label{sec:exp}
 
Before describing the experiments carried out, notice that in the \textsc{cts-gaussian} policies, $\beta>1$ is an artefact of the analysis and can in practice be taken equal to $1$. This is what we did in our experiments.
 
 \paragraph{The shortest path problem}
 We compare our \textsc{cts} policies to \textsc{cucb} and \textsc{cucb-kl}, for the shortest path problem on the road chesapeake network \citep{roadnet}. This network contains $39$ nodes and $n=170$ edges. $\cA$ is the set of paths from an origin to a destination in the network. 
 We choose a linear reward, so that an efficient $\ora$ exists for this problem. We choose $\bmu
^*$ uniformly in $[-1,0]^n$ and then normalize its sum so that 
 $\sum_i \mu_i^*=-s$, where $s$ is unknown to the agent. The parameter
 $s$ stands for the global network traffic (e.g., the total number of vehicles in the network).
 We run two experiments, one with $-\bX\sim\otimes_i \Bernoulli\pa{-\mu_i^*}$ and another with $-\bX\sim\otimes_i \Bernoulli\pa{-\mu_i^*}$ conditionally on $\sum_i X_i = -s$. They are presented in Figure~\ref{exp:shortest}. Since the outcomes are not mutually independent in this last experiment, we use \textsc{(clip) cts-gaussian} rather than \textsc{cts-beta}, where we take $D_i=1/4$, using that for any $\blambda\in\R_+^n$, $\EE{e^{ \blambda\transpose \bX }}\leq\prod_{i\in [n]}\EE{e^{\lambda_iX_i}} $ (see e.g., \citet{borcea2009negative}, corollary 4.18).
 It is clear from the experiments that \textsc{cts} policies outperform both \textsc{cucb} and  \textsc{cucb-kl}. In the second experiment, we see that \textsc{clip cts-gaussian} and \textsc{cts-gaussian} are very similar --- which is not surprising because $D_i$ is not large here (unlike in the next experiment) --- and that for a small $s$, \textsc{cucb-kl} becomes competitive, since the $\mathrm{kl}$ is much larger than the quadratic divergence in that case.
  \vspace{-.15cm}
 \paragraph{Comparison to \textsc{escb} for the matching problem}\label{subsec:exp_matching}
 We consider here a comparison between \textsc{(clip) cts-gaussian}, \textsc{cucb} and \textsc{escb} (we refer the reader to \citet{Wang2018} for a comparison between \textsc{cts-beta} and \textsc{escb}). Since \textsc{escb} is computationally intractable, we limit ourselves to a toy  matching problem  on the complete bipartite graphs $K_{4,4}$, with $\bX\sim\cN(\bmu^*,(c\II{i\neq j} +\II{i=j})_{ij}$), where this covariance is known to the agent. Our results are shown in Figure~\ref{exp:max_matching}, where we observe that \textsc{clip cts-gaussian} (resp. \textsc{escb})  is slightly better for $c$ small (resp. large), thus reaching the best of both worlds. This is because a large $c$ forces \textsc{clip cts-gaussian} to oversample (as evidenced by \textsc{cts-gaussian} whose performance is even worse than \textsc{cucb} for $c=1$). We also recorded the computation time for larger instances (see Table~\ref{table:comp_time}), and observe the efficiency of \textsc{cucb} and \textsc{clip cts-gaussian} compared to \textsc{escb}.
 
   \paragraph{Correlated vs independent prior in practice} We briefly discussed the use of a correlated prior in footnote~\ref{note3}, with covariance $\pa{C_{ij}N_{ij,t-1}N^{-1}_{i,t-1}N^{-1}_{j,t-1}}_{ij}$, mentioning that the policy would perform better than using an independent prior. We ran additional empirical comparisons to assess this, plotting the results in Figure~\ref{fig:add_exp} where we also compared with a common prior policy approach \citep{agrawal2017thompson}, i.e., with covariance $\pa{{N^{-1/2}_{i,t-1}N^{-1/2}_{j,t-1}}}_{ij}$.\footnote{We also tried the policy (without displaying the results, for the sake of clarity) with covariance $\pa{C_{ij}{N^{-1/2}_{i,t-1}N^{-1/2}_{j,t-1}}}_{ij}$, and observed about the same performance as the correlated prior approach.} As expected, the correlated prior policy is better than the independent one (when outcomes are correlated). This motivates the theoretical study of such policy for future work. The  common prior approach is comparable to the correlated prior one on the matching problem, but it is outperformed in the worst-case scenario of a separate action space $\cA=\set{\set{km+1,\dots, (k+1)m} \mid k \in \set{0,\dots, \frac{n}{m}-1}}$ with independent outcomes. This is because such problem reduces to a classical MAB problem with a covariance scaled up by a factor $m$, whereas the common prior approach has a variance scaled up by a factor $m^2$.

 \begin{figure}[t]
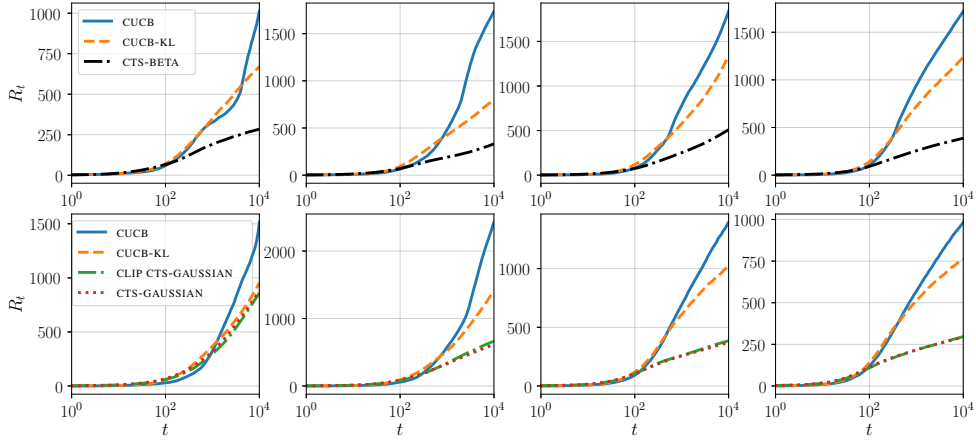

\centering
\begin{adjustbox}{clip,trim=0cm .8cm 0cm .6cm}
\resizebox{\textwidth}{!}{\input{multiple_regret_170_10000_0.pgf}}
\end{adjustbox}
\begin{adjustbox}{clip,trim=0cm .6cm 0cm .6cm}
\resizebox{\textwidth}{!}{\input{multiple_regret_170_10000_0c.pgf}}
\end{adjustbox}
\caption{Cumulative regret (averaged over 50 simulations) for the shortest path problem.
\textbf{Top:} with mutually independent outcomes, taking the opposite sum of means being $s=70,90,110,130$ respectively.
\textbf{Bottom:} with correlated outcomes, taking the opposite sum of outcomes being $s=70,90,110,130$ respectively.}
\label{exp:shortest}
\end{figure}

\begin{figure}[t]
\centering
\begin{adjustbox}{clip,trim=0cm .6cm 0cm .6cm}
\resizebox{\textwidth}{!}{\input{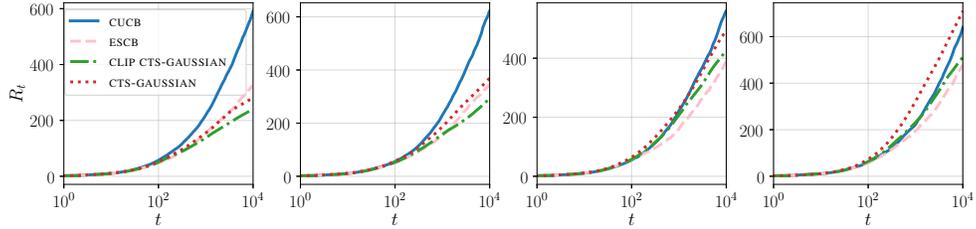}}
\end{adjustbox}
\caption{Cumulative regret (averaged over 50 simulations) for the matching problem with Gaussian outcomes, taking $c=-1/n,0.2,0.5,1$ respectively.}
\label{exp:max_matching}
\end{figure}

\begin{table}[H]
  \caption{Computation time per round (ms), with $c=0.3$, $T=100$, averaged over $5$ simulations.}
  \centering
  \begin{tabular}{c|cccccc}
    \toprule
         & $K_{3,3}$     & $K_{4,4}$ & $K_{5,5}$& $K_{6,6}$ & $K_{7,7}$ &
         $K_{8,8}$\\
    \midrule
    \textsc{cucb} & $0.39$  & $0.64$ &$1.23$ &$1.65$  & $2.45$ & $3.88$ \\
    \textsc{clip cts-gaussian}     & $0.50$& $0.80$ & $1.75$ & $1.79$ & $3.30$ & $5.42$   \\
    \textsc{escb}     & $0.45$& $1.93$ & $10.3$ & $75.6$ & $541$ & $4694$\\
    \bottomrule
  \end{tabular}\label{table:comp_time}
\end{table}

  \begin{figure}[H]
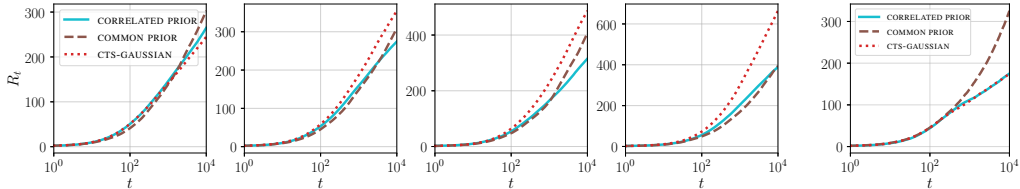

    \centering
    \begin{adjustbox}{clip,trim=0cm 0cm 0cm 0cm}
    \resizebox{\textwidth}{!}{\input{regret_multiple_MNLandcorr_16_10000_1.pgf}\input{regret_20_10000_4_0_2.pgf}}
    \end{adjustbox}
    \caption{\label{fig:add_exp}Comparison with correlated prior sampling and common prior sampling (averaged over 50 simulations). \textbf{The first 4:} for the $K_{4,4}$ matching problem, with Gaussian outcomes, taking $c=0,0.2,0.5,1$. \textbf{The last:} for $\cA=\set{\set{km+1,\dots, (k+1)m} \mid k \in \set{0,\dots, \frac{n}{m}-1}}$, $c=0$.}
\end{figure}
 
  \section{Conclusion and future work}
 In this paper, we have provided the first efficient policies having an optimal regret bound for a wide spectrum of problems instances for CMAB with semi-bandit feedback. Our approach also answers the question of finding an analysis for \textsc{cts} under correlated arm distributions. There are several possible extensions that could be considered as future work. For example,  it would be interesting to have an analysis of \textsc{cts} with a \emph{correlated} (Gaussian) prior. Indeed, apart from the empirical gain, this would open up the possibility of estimating the covariance matrix and using it in the prior distribution.  
 Further relevant results would be an analysis of \textsc{cts-beta} without the mutual independence of outcomes, or also an improved concentration bound for a sum of independent betas, relying on the $\mathrm{kl}$ rather than using sub-Gaussianity. This latter result would thus show that \textsc{cts-beta} dominates \textsc{cucb-kl}, which is empirically observed.

\clearpage
\section*{Broader Impact}

This work does not present any foreseeable societal consequence.

\begin{ack}
\vspace{-.1cm}
The research presented was supported by European CHIST-ERA project DELTA, French Ministry of
Higher Education and Research, Nord-Pas-de-Calais Regional Council,  French National Research Agency project BOLD (ANR19-CE23-0026-04).

It was also supported in part by a public grant as part of the Investissement d'avenir project, reference ANR-11-LABX-0056-LMH, LabEx LMH, in a joint call with Gaspard Monge Program for optimization, operations research and their interactions with data sciences.

\end{ack}
 \vspace{-.1cm}

\setlength{\bibsep}{5.1pt}
\bibliography{example}

\begin{thebibliography}{41}
\providecommand{\natexlab}[1]{#1}
\providecommand{\url}[1]{\texttt{#1}}
\expandafter\ifx\csname urlstyle\endcsname\relax
  \providecommand{\doi}[1]{doi: #1}\else
  \providecommand{\doi}{doi: \begingroup \urlstyle{rm}\Url}\fi

\bibitem[Agrawal and Goyal(2012)]{agrawal2012thompsonarxiv}
S.~Agrawal and N.~Goyal.
\newblock {Thompson Sampling for Contextual Bandits with Linear Payoffs}.
\newblock \emph{CoRR, abs/1209.3352, http://arxiv.org/abs/1209.3352}, sep 2012.
\newblock URL \url{http://arxiv.org/abs/1209.3352}.

\bibitem[Agrawal et~al.(2017)Agrawal, Avadhanula, Goyal, and
  Zeevi]{agrawal2017thompson}
S.~Agrawal, V.~Avadhanula, V.~Goyal, and A.~Zeevi.
\newblock Thompson sampling for the mnl-bandit.
\newblock \emph{arXiv preprint arXiv:1706.00977}, 2017.

\bibitem[Anantharam et~al.(1987)Anantharam, Varaiya, and
  Walrand]{anantharam1987asymptotically}
V.~Anantharam, P.~Varaiya, and J.~Walrand.
\newblock Asymptotically efficient allocation rules for the multiarmed bandit
  problem with multiple plays-part i: Iid rewards.
\newblock \emph{IEEE Transactions on Automatic Control}, 32\penalty0
  (11):\penalty0 968--976, 1987.

\bibitem[Atamt{\"u}rk and G{\'o}mez(2017)]{atamturk2017maximizing}
A.~Atamt{\"u}rk and A.~G{\'o}mez.
\newblock Maximizing a class of utility functions over the vertices of a
  polytope.
\newblock \emph{Operations Research}, 65\penalty0 (2):\penalty0 433--445, 2017.

\bibitem[Auer et~al.(2002)Auer, Cesa-Bianchi, and Fischer]{auer2002finite}
P.~Auer, N.~Cesa-Bianchi, and P.~Fischer.
\newblock {Finite-time analysis of the multiarmed bandit problem}.
\newblock \emph{Machine Learning}, 47\penalty0 (2-3):\penalty0 235--256, 2002.

\bibitem[Berry and Fristedt(1985)]{berry1985bandit}
D.~A. Berry and B.~Fristedt.
\newblock \emph{{Bandit Problems: Sequential Allocation of Experiments}},
  volume~38 of \emph{Monographs on statistics and applied probability}.
\newblock Chapman and Hall, 1985.

\bibitem[Borcea et~al.(2009)Borcea, Br{\"a}nd{\'e}n, and
  Liggett]{borcea2009negative}
J.~Borcea, P.~Br{\"a}nd{\'e}n, and T.~Liggett.
\newblock Negative dependence and the geometry of polynomials.
\newblock \emph{Journal of the American Mathematical Society}, 22\penalty0
  (2):\penalty0 521--567, 2009.

\bibitem[Cesa-Bianchi and Lugosi(2012)]{cesa-bianchi2012combinatorial}
N.~Cesa-Bianchi and G.~Lugosi.
\newblock {Combinatorial bandits}.
\newblock In \emph{Journal of Computer and System Sciences}, volume~78, pages
  1404--1422, 2012.

\bibitem[Chang et~al.(2011)Chang, Cosman, and Milstein]{chang2011chernoff}
S.-H. Chang, P.~C. Cosman, and L.~B. Milstein.
\newblock Chernoff-type bounds for the gaussian error function.
\newblock \emph{IEEE Transactions on Communications}, 59\penalty0
  (11):\penalty0 2939--2944, 2011.

\bibitem[Chen et~al.(2013)Chen, Wang, and Yuan]{chen13a}
W.~Chen, Y.~Wang, and Y.~Yuan.
\newblock Combinatorial multi-armed bandit: General framework and applications.
\newblock In S.~Dasgupta and D.~McAllester, editors, \emph{Proceedings of the
  30th International Conference on Machine Learning}, volume~28 of
  \emph{Proceedings of Machine Learning Research}, pages 151--159, Atlanta,
  Georgia, USA, 17--19 Jun 2013. PMLR.
\newblock URL \url{http://proceedings.mlr.press/v28/chen13a.html}.

\bibitem[Chen et~al.(2016)Chen, Wang, and Yuan]{Chen2015combinatorial}
W.~Chen, Y.~Wang, and Y.~Yuan.
\newblock {Combinatorial multi-armed bandit and its extension to
  probabilistically triggered arms}.
\newblock \emph{Journal of Machine Learning Research}, 17, 2016.

\bibitem[Combes et~al.(2015)Combes, Shahi, Proutiere, and
  Others]{combes2015combinatorial}
R.~Combes, M.~S. T.~M. Shahi, A.~Proutiere, and Others.
\newblock {Combinatorial bandits revisited}.
\newblock In \emph{Advances in Neural Information Processing Systems}, pages
  2116--2124, 2015.

\bibitem[Cuvelier et~al.(2020)Cuvelier, Combes, and
  Gourdin]{cuvelier2020statistically}
T.~Cuvelier, R.~Combes, and E.~Gourdin.
\newblock Statistically efficient, polynomial time algorithms for combinatorial
  semi bandits.
\newblock \emph{arXiv preprint arXiv:2002.07258}, 2020.

\bibitem[Degenne and Perchet(2016{\natexlab{a}})]{Anytime}
R.~Degenne and V.~Perchet.
\newblock Anytime optimal algorithms in stochastic multi-armed bandits.
\newblock In M.~F. Balcan and K.~Q. Weinberger, editors, \emph{Proceedings of
  The 33rd International Conference on Machine Learning}, volume~48 of
  \emph{Proceedings of Machine Learning Research}, pages 1587--1595, New York,
  New York, USA, 20--22 Jun 2016{\natexlab{a}}. PMLR.
\newblock URL \url{http://proceedings.mlr.press/v48/degenne16.html}.

\bibitem[Degenne and Perchet(2016{\natexlab{b}})]{Degenne2016}
R.~Degenne and V.~Perchet.
\newblock Combinatorial semi-bandit with known covariance.
\newblock In D.~D. Lee, M.~Sugiyama, U.~V. Luxburg, I.~Guyon, and R.~Garnett,
  editors, \emph{Advances in Neural Information Processing Systems 29}, pages
  2972--2980. Curran Associates, Inc., 2016{\natexlab{b}}.
\newblock URL
  \url{http://papers.nips.cc/paper/6137-combinatorial-semi-bandit-with-known-covariance.pdf}.

\bibitem[Durrett(2019)]{durrett2019probability}
R.~Durrett.
\newblock \emph{Probability: theory and examples}, volume~49.
\newblock Cambridge university press, 2019.

\bibitem[Gai et~al.(2010)Gai, Krishnamachari, and Jain]{gai2010learning}
Y.~Gai, B.~Krishnamachari, and R.~Jain.
\newblock Learning multiuser channel allocations in cognitive radio networks: A
  combinatorial multi-armed bandit formulation.
\newblock In \emph{2010 IEEE Symposium on New Frontiers in Dynamic Spectrum
  (DySPAN)}, pages 1--9. IEEE, 2010.

\bibitem[Gai et~al.(2012)Gai, Krishnamachari, and Jain]{gai2012combinatorial}
Y.~Gai, B.~Krishnamachari, and R.~Jain.
\newblock {Combinatorial network optimization with unknown variables:
  Multi-armed bandits with linear rewards and individual observations}.
\newblock \emph{Transactions on Networking}, 20\penalty0 (5):\penalty0
  1466--1478, 2012.

\bibitem[Gopalan et~al.(2014)Gopalan, Mannor, and Mansour]{gopalan2013thompson}
A.~Gopalan, S.~Mannor, and Y.~Mansour.
\newblock {Thompson sampling for complex bandit problems}.
\newblock In \emph{International Conference on Machine Learning}, 2014.

\bibitem[Hoeffding(1963)]{hoeffding1963probability}
W.~Hoeffding.
\newblock {Probability inequalities for sums of bounded random variables}.
\newblock \emph{Journal of the American Statistical Association}, 58:\penalty0
  13--30, 1963.

\bibitem[Jacobs and Wozencraft(1965)]{jacobs1965principles}
I.~M. Jacobs and J.~Wozencraft.
\newblock Principles of communication engineering.
\newblock 1965.

\bibitem[Kaufmann et~al.(2012)Kaufmann, Korda, and Munos]{kaufmann2012thompson}
E.~Kaufmann, N.~Korda, and R.~Munos.
\newblock {Thompson Sampling: An Asymptotically Optimal Finite Time Analysis}.
\newblock \emph{Algorithmic Learning Theory}, 2012.

\bibitem[Komiyama et~al.(2015)Komiyama, Honda, and Nakagawa]{Komiyama2015}
J.~Komiyama, J.~Honda, and H.~Nakagawa.
\newblock {Optimal Regret Analysis of Thompson Sampling in Stochastic
  Multi-armed Bandit Problem with Multiple Plays}.
\newblock jun 2015.
\newblock URL \url{http://arxiv.org/abs/1506.00779}.

\bibitem[Kveton et~al.(2014)Kveton, Wen, Ashkan, Eydgahi, and
  Eriksson]{kveton2014matroid}
B.~Kveton, Z.~Wen, A.~Ashkan, H.~Eydgahi, and B.~Eriksson.
\newblock {Matroid bandits: Fast combinatorial optimization with learning}.
\newblock In \emph{Uncertainty in Artificial Intelligence}, 2014.

\bibitem[Kveton et~al.(2015)Kveton, Wen, Ashkan, and
  Szepesvari]{kveton2015tight}
B.~Kveton, Z.~Wen, A.~Ashkan, and C.~Szepesvari.
\newblock {Tight regret bounds for stochastic combinatorial semi-bandits}.
\newblock In \emph{International Conference on Artificial Intelligence and
  Statistics}, 2015.

\bibitem[Lai and Robbins(1985)]{lai1985asymptotically}
T.~L. Lai and H.~Robbins.
\newblock {Asymptotically efficient adaptive allocation rules}.
\newblock \emph{Advances in Applied Mathematics}, 6\penalty0 (1):\penalty0
  4--22, 1985.

\bibitem[Liu and Zhao(2012)]{liu2012adaptive}
K.~Liu and Q.~Zhao.
\newblock Adaptive shortest-path routing under unknown and stochastically
  varying link states.
\newblock In \emph{2012 10th International Symposium on Modeling and
  Optimization in Mobile, Ad Hoc and Wireless Networks (WiOpt)}, pages
  232--237. IEEE, 2012.

\bibitem[Marchal et~al.(2017)Marchal, Arbel, et~al.]{marchal2017sub}
O.~Marchal, J.~Arbel, et~al.
\newblock On the sub-gaussianity of the beta and dirichlet distributions.
\newblock \emph{Electronic Communications in Probability}, 22, 2017.

\bibitem[Perrault et~al.(2019{\natexlab{a}})Perrault, Perchet, and
  Valko]{perrault2019exploiting}
P.~Perrault, V.~Perchet, and M.~Valko.
\newblock Exploiting structure of uncertainty for efficient matroid
  semi-bandits.
\newblock In K.~Chaudhuri and R.~Salakhutdinov, editors, \emph{Proceedings of
  the 36th International Conference on Machine Learning}, volume~97 of
  \emph{Proceedings of Machine Learning Research}, pages 5123--5132, Long
  Beach, California, USA, 09--15 Jun 2019{\natexlab{a}}. PMLR.
\newblock URL \url{http://proceedings.mlr.press/v97/perrault19a.html}.

\bibitem[Perrault et~al.(2019{\natexlab{b}})Perrault, Perchet, and
  Valko]{pmlr-v89-perrault19a}
P.~Perrault, V.~Perchet, and M.~Valko.
\newblock Finding the bandit in a graph: Sequential search-and-stop.
\newblock In K.~Chaudhuri and M.~Sugiyama, editors, \emph{Proceedings of
  Machine Learning Research}, volume~89 of \emph{Proceedings of Machine
  Learning Research}, pages 1668--1677. PMLR, 16--18 Apr 2019{\natexlab{b}}.
\newblock URL \url{http://proceedings.mlr.press/v89/perrault19a.html}.

\bibitem[Perrault et~al.(2020{\natexlab{a}})Perrault, Healey, Wen, and
  Valko]{perraultbudgeted2020}
P.~Perrault, J.~Healey, Z.~Wen, and M.~Valko.
\newblock Budgeted online influence maximization.
\newblock In \emph{Proceedings of the 37th International Conference on Machine
  Learning}, pages 6588--6599. 2020{\natexlab{a}}.

\bibitem[Perrault et~al.(2020{\natexlab{b}})Perrault, Perchet, and
  Valko]{perrault2020covariance-adapting}
P.~Perrault, V.~Perchet, and M.~Valko.
\newblock {Covariance-adapting algorithm for semi-bandits with application to
  sparse rewards}.
\newblock In \emph{Conference on Learning Theory}, 2020{\natexlab{b}}.

\bibitem[Robbins(1952)]{robbins1952some}
H.~Robbins.
\newblock {Some aspects of the sequential design of experiments}.
\newblock \emph{Bulletin of the American Mathematics Society}, 58:\penalty0
  527--535, 1952.

\bibitem[Rossi and Ahmed(2015)]{roadnet}
R.~A. Rossi and N.~K. Ahmed.
\newblock The network data repository with interactive graph analytics and
  visualization.
\newblock In \emph{AAAI}, 2015.
\newblock URL \url{http://networkrepository.com}.

\bibitem[Russo and Van~Roy(2016)]{russo2016information}
D.~Russo and B.~Van~Roy.
\newblock An information-theoretic analysis of thompson sampling.
\newblock \emph{The Journal of Machine Learning Research}, 17\penalty0
  (1):\penalty0 2442--2471, 2016.

\bibitem[Shaked and Shanthikumar(2007)]{shaked2007stochastic}
M.~Shaked and J.~G. Shanthikumar.
\newblock \emph{Stochastic orders}.
\newblock Springer Science \& Business Media, 2007.

\bibitem[Talebi and Proutiere(2016)]{Talebi2016}
M.~S. Talebi and A.~Proutiere.
\newblock {An Optimal Algorithm for Stochastic Matroid Bandit Optimization}.
\newblock In \emph{The 2016 International Conference on Autonomous Agents {\&}
  Multiagent Systems}, pages 548--556, 2016.
\newblock ISBN 9781450342391.

\bibitem[Thompson(1933)]{thompson1933likelihood}
W.~R. Thompson.
\newblock {On the likelihood that one unknown probability exceeds another in
  view of the evidence of two samples}.
\newblock \emph{Biometrika}, 25:\penalty0 285--294, 1933.

\bibitem[Wang and Chen(2017)]{wang2017improving}
Q.~Wang and W.~Chen.
\newblock {Improving regret bounds for combinatorial semi-bandits with
  probabilistically triggered arms and its applications}.
\newblock In \emph{Neural Information Processing Systems}, mar 2017.
\newblock URL \url{http://arxiv.org/abs/1703.01610}.

\bibitem[Wang and Chen(2018)]{Wang2018}
S.~Wang and W.~Chen.
\newblock {Thompson Sampling for Combinatorial Semi-Bandits}.
\newblock mar 2018.
\newblock URL \url{http://arxiv.org/abs/1803.04623}.

\bibitem[Wen et~al.(2015)Wen, Kveton, and Ashkan]{wen2015efficient}
Z.~Wen, B.~Kveton, and A.~Ashkan.
\newblock Efficient learning in large-scale combinatorial semi-bandits.
\newblock In \emph{International Conference on Machine Learning}, pages
  1113--1122, 2015.

\end{thebibliography}

\clearpage
\appendix

\section{Proof of Theorem~\ref{thm:tsbeta}}\label{app:thm:tsbeta}
We first restate the complete non-asymptotic upper-bound as follows.
\begin{theorem*}
The policy $\pi$ described in
 Algorithm~\ref{algo:tsbeta} has regret $R_T(\pi)$
 bounded by
\begin{align*}
 16\log_2^2(16m)\!\sum_{i\in [n]}\!\frac{{B^2\!\log\pa{2^m\abs{\cA}T}}}{\Delta_{i,\min}}\!+\!
\Delta_{\max}(1+n)\!+\!\frac{nm^2\Delta_{\max}}{\pa{\frac{\Delta_{\min}}{2B}-({m^*}^2+1)\eps}^{2}}\!+\!\Delta_{\max}\frac{C}{\eps^2}\pa{\frac{C'}{\eps^4}}^{m^*}\!\!,
\end{align*}
where $C,C'$ are two universal constants, and $\eps\in (0,1)$ is such that ${\Delta_{\min}}/(2B)-({m^*}^2+1)\eps>~0.$
\end{theorem*}
\subsection{Preliminary lemmas}
In order to prove Theorem~\ref{thm:tsbeta}, we modify two lemmas from \cite{Wang2018}: first, in their Lemma~3, we replace $\eps$ by ${\Delta_{\min}}/\pa{2B}-({m^*}^2+1)\eps>0,$
which gives the following Lemma~\ref{lem:nlem3}. 
\begin{lemma}\label{lem:nlem3}In Algorithm~\ref{algo:tsbeta}, for any arm $i$, we have
 \[\EE{\abs{t\in [T],~i\in A_t,~\abs{A_t}\cdot\abs{\bar\mu_{i,t-1}-\mu_i^*}>\frac{\Delta_{\min}}{2B}-({m^*}^2+1)\eps}}\leq 1+\pa{\frac{\Delta_{\min}}{2mB}-\frac{({m^*}^2+1)\eps}{m}}^{-2}.\]
\end{lemma}
Then, we modify Lemma~4 from \cite{Wang2018} as follows, leveraging on the mutual independence of $\theta_{1,t},\dots,\theta_{n,t}$ to get a tighter confidence region for the sample $\btheta_t$. 
\begin{lemma}\label{lem:nlem4}
 In Algorithm~\ref{algo:tsbeta}, for all round $t$, we have
 \[\PPc{\norm{\be_{A_t}\odot\pa{\btheta_{t}-\bar\bmu_{t-1}} }_1\geq\sqrt{\frac{1}{2}{\log\pa{\abs{\cA}2^mT}}\sum_{i\in {A_t}}\frac{1}{N_{i,t-1}}}}{\cH_t}\leq 1/T.\]
\end{lemma}
\begin{proof}
From \citep{marchal2017sub}, the Beta random variable from $\theta_{i,t}$ is sub-Gaussian with variance $1/(4 N_{i,t-1})$.
Thus, defining the functions \[\alpha_t(A)\triangleq\sqrt{\frac{1}{2}{\log\pa{\abs{\cA}2^mT}}\sum_{i\in {A}}\frac{1}{N_{i,t-1}}},\quad\text{and}\quad \lambda_t(A) \triangleq \frac{4\alpha_t(A)}{\sum_{i\in A}1/N_{i,t-1}}, \] we have 
\begin{align*}\PPc{\norm{\be_{A_t}\odot\pa{\btheta_{t}-\bar\bmu_{t-1}} }_1\geq \alpha_t(A_t)}{\cH_t}&\leq\sum_{A\in \cA}\PPc{\norm{\be_{A}\odot\pa{\btheta_{t}-\bar\bmu_{t-1}} }_1\geq \alpha_t(A)}{\cH_t}\\&\leq \sum_{A\in \cA} e^{-\lambda_t(A) \alpha_t(A)}\EEc{e^{\lambda_t(A) \norm{\be_{A}\odot\pa{\btheta_{t}-\bar\bmu_{t-1}} }_1}}{\cH_t}
\\&\leq 
\sum_{A\in \cA} e^{-\lambda_t(A) \alpha_t(A)}\prod_{i\in A}\EEc{e^{\lambda_t(A) \abs{{\theta_{i,t}-\bar\mu_{i,t-1}} }}}{\cH_t}
\\&\leq \sum_{A\in \cA} e^{-\lambda_t(A) \alpha_t(A)}\prod_{i\in A}\EEc{e^{\lambda_t(A) \pa{{\theta_{i,t}-\bar\mu_{i,t-1}} }}+e^{\lambda_t(A) \pa{{\bar\mu_{i,t-1} -\theta_{i,t}} }}}{\cH_t}
\\&\leq \sum_{A\in \cA} 2^{\abs{A}}e^{-\lambda_t(A) \alpha_t(A)}{e^{\lambda_t(A)^2 {\sum_{i\in A}1/(8 N_{i,t-1}) }}}\leq 1/T.
\end{align*}
\end{proof}
\subsection{Main proof}
With the two lemmas from the previous subsection, we are ready to demonstrate Theorem~\ref{thm:tsbeta}.
We consider the following events.
\begin{itemize}
    \item $\fZ_t\triangleq\set{\Delta_t>0}$
    \item $\fB_t\triangleq \set{\exists i\in A_t,~{\abs{A_t}}\cdot\abs{\bar\mu_{i,t-1}-\mu_i^*}>{{\Delta_{\min}}/\pa{2B}-({m^*}^2+1)\eps} }$
    \item $ \fC_t\triangleq\set{{\norm{\be_{A_t}\odot\pa{\btheta_{t}-\bmu^*} }_1}> \Delta_t/B-\pa{{m^*}^2+1}\eps}$
    \item $\fD_t\triangleq \set{\norm{\be_{A_t}\odot\pa{\btheta_{t}-\bar\bmu_{t-1}} }_1\geq\sqrt{{0.5}\cdot{\log\pa{\abs{\cA}2^mT}}\sum_{i\in {A_t}}{1}/{N_{i,t-1}}}}$.
\end{itemize}

We break down our analysis into 4 steps. The main novelties are in the last two steps: Step 3 gives us the tighter dependence in $m$, and Step 4, that contains the main difficulties, gives the new exponential constant term.

\paragraph{Step 1: bound under $\fZ_t\wedge\fB_t$} By Lemma~\ref{lem:nlem3},
\begin{align*}\sum_{t\in [T]}\EE{\Delta_t\II{\fZ_t\wedge\fB_t}}&\leq \Delta_{\max}\sum_{i\in [n]}\EE{\abs{t\in [T],~i\in A_t,~\abs{A_t}\cdot\abs{\bar\mu_{i,t-1}-\mu_i^*}>{\Delta_{\min}}/(2B)-({m^*}^2+1)\eps}}\\&\leq n\Delta_{\max}\pa{1+\pa{\frac{\Delta_{\min}}{2mB}-\frac{({m^*}^2+1)\eps}{m}}^{-2}}.\end{align*}
\paragraph{Step 2: bound under $\fZ_t\wedge\neg\fB_t\wedge\fC_t\wedge\fD_t$} By Lemma~\ref{lem:nlem4},
\begin{align*}\sum_{t\in [T]}\EE{\Delta\pa{A_t}\II{\fZ_t\wedge\neg\fB_t\wedge\fC_t\wedge\fD_t}}&\leq \Delta_{\max}\sum_{t\in [T]}\EE{\PPc{\fD_t}{\cH_t}}\leq \Delta_{\max}\sum_{t\in [T]}1/T=\Delta_{\max}.\end{align*}

\paragraph{Step 3: bound under $\fZ_t\wedge\neg\fB_t\wedge\fC_t\wedge\neg\fD_t$}

\begin{align*}
  \Delta_t/B&\leq \norm{\be_{A_t}\odot\pa{\btheta_{t}-\bmu^*} }_1+\pa{{m^*}^2+1}\eps &\fC_t
   \\&\leq \norm{\be_{A_t}\odot\pa{\btheta_{t}-\bar\bmu_{t-1}} }_1+\norm{\be_{A_t}\odot\pa{\bar\bmu_{t-1}-\bmu^*} }_1+\pa{{m^*}^2+1}\eps
\\&\leq \norm{\be_{A_t}\odot\pa{\btheta_{t}-\bar\bmu_{t-1}} }_1+\Delta_{\min}/(2B)-\pa{{m^*}^2+1}\eps+\pa{{m^*}^2+1}\eps&\neg\fB_t
\\&\leq \norm{\be_{A_t}\odot\pa{\btheta_{t}-\bar\bmu_{t-1}} }_1+\Delta_t/(2B)&\fZ_t
\\&\leq \sqrt{\frac{1}{2}{\log\pa{\abs{\cA}2^mT}}\sum_{i\in {A_t}}\frac{1}{N_{i,t-1}}} +\Delta_t/(2B).&\neg\fD_t
\end{align*}

So we have that the following event holds 
\[\fA_t\triangleq\set{ \Delta_t\leq B\sqrt{{2}{\log\pa{\abs{\cA}2^mT}}\sum_{i\in {A_t}}\frac{1}{N_{i,t-1}}}}.\]
We can thus apply Theorem~\ref{thm:dege} (see Appendix~\ref{app:gen}) to get the bound
\begin{align*}\sum_{t\in [T]}\EE{\Delta_t\II{\fZ_t,\neg\fB_t,\fC_t,\neg\fD_t}}&\leq
\sum_{t\in [T]}\EE{\Delta_t\II{\fA_t}}\\&\leq
32B^2\log_2^2(4\sqrt{m})\sum_{i\in [n]}\Delta_{i,\min}^{-1}2{\log\pa{\abs{\cA}2^mT}}.\end{align*}
\paragraph{Step 4: bound under $\fZ_t\wedge\neg\fC_t$}
We consider the following events for a subset $Z\subset [n]$ 
\[\fR(\btheta',Z)\triangleq \set{Z\subset \ora\pa{\btheta'},~\norm{\be_{\ora\pa{\btheta'}}\odot\pa{\btheta'-\bmu^*}}_1>\Delta\pa{\ora\pa{\btheta'}}-({k^*}^2+1)\eps}\]
\begin{align}\fS_t\pa{Z}\triangleq \set{\forall \btheta' \text{ s.t. } \norm{\pa{\bmu^*-\btheta'}\odot\be_{Z}}_\infty\leq \eps,~ \fR(\btheta'\odot\be_{Z}+\btheta_t\odot\be_{Z^c},Z) \text{ holds} }\label{rel:llaass}\end{align}
\begin{align*}\fT_t\pa{Z}\triangleq \set{\norm{\pa{\bmu^*-\btheta_t}\odot\be_{Z}}_\infty> \eps}.\label{rel:llaassd}\end{align*}
We can state the three following lemmas. Note that Lemma~\ref{lem:fSfT} is exactly the Lemma~1 from \citet{Wang2018}. The other two replace their Lemma~7. 
\begin{lemma}In Algorithm~\ref{algo:tsbeta}, for all round $t$, we have
 \[\fZ_t,\neg \fC_t\imp \exists Z\subset A^*,~Z\neq \emptyset~\text{s.t. the event }\fS_t\pa{Z}\wedge\fT_t\pa{Z}\text{ holds.}\]
 \label{lem:fSfT}
\end{lemma}
\begin{lemma} Given $Z\subset A^*,~Z\neq \emptyset$, let $\tau_q$ be  the  round  at  which $\fS_t\pa{Z}\wedge\neg\fT_t\pa{Z}$ occurs for the $q$-th time, and let $\tau_0= 0$. Then, in Algorithm~\ref{algo:tsbeta}, we have
\[ \EE{\sum_{t=\tau_q+1}^{\tau_{q+1}} \II{\fS_t\pa{Z},\fT_t\pa{Z}}}\leq \EE{\sup_{\tau\geq{\tau_q+1}}\prod_{i\in Z}\frac{1}{\PPc{\abs{\theta_{i,\tau}-\mu_i^*}\leq \eps}{\cH_\tau}}}-1.\]\label{lem:Esup}
\end{lemma}
\begin{lemma}\label{lem:Esupbound}In Algorithm~\ref{algo:tsbeta}, we have
\[\EE{\sup_{\tau\geq{\tau_q+1}}\prod_{i\in Z}\frac{1}{\PPc{\abs{\theta_{i,\tau}-\mu_i^*}\leq \eps}{\cH_\tau}}}-1 \leq \left\{
    \begin{array}{ll}
        \pa{  c\eps^{-4}}^{\abs{Z}} & \mbox{for every } q\geq 0 \\
        e^{-\eps^2 q/8}\pa{  c'\eps^{-4}}^{\abs{Z}}& \mbox{if }  q> 8/\eps^2,
    \end{array}
\right. \]
where $c$ and $c'$ are two universal constants.
\end{lemma}
These lemmas allow us to get a constant regret under the event $\fZ_t\wedge\neg\fC_t$. Indeed, we have from Lemma~\ref{lem:fSfT} that
\begin{align*}\sum_{t\in [T]}\EE{\Delta_t\II{\fZ_t\wedge\neg\fC_t}}&\leq \Delta_{\max}\sum_{Z\subset A^*,~Z\neq \emptyset}\EE{\sum_{t\in [T]}\II{\fS_t(Z)\wedge\fT_t(Z)}}\\&=\Delta_{\max}\sum_{Z\subset A^*,~Z\neq \emptyset}\sum_{q\geq 0} \EE{\sum_{t=\tau_q+1}^{\tau_{q+1}} \II{\fS_t\pa{Z},\fT_t\pa{Z}}}.\end{align*}
Lemma~\ref{lem:Esup} and \ref{lem:Esupbound} gives that the above is further upper bounded by
\[\Delta_{\max}\sum_{Z\subset A^*,~Z\neq \emptyset}\pa{\sum_{q= 0}^{\ceil{8/\eps^2}-1}\pa{  c\eps^{-4}}^{\abs{Z}} + \sum_{q\geq \ceil{8/\eps^2}}e^{-\eps^2 q/8}\pa{  c'\eps^{-4}}^{\abs{Z}}}\]
which is bounded by
\[\Delta_{\max}\frac{C}{\eps^2}\pa{\frac{C'}{\eps^4}}^{m^*},\]
where $C$ and $C'$ are two universal constants.
This concludes the proof of the theorem.
\begin{proof}[Proof of Lemma~\ref{lem:Esup}]
 Since $\fS_t\pa{Z},\fT_t\pa{Z}$ are independent conditioned on the history $\cH_t$, the LHS is
\[\EE{\sum_{k\geq 1} (k-1) \PPc{\neg\fT_{t_{k,q}}\pa{Z}}{\cH_{t_{k,q}}} \prod_{j=1}^{k-1}\PPc{\fT_{t_{j,q}}\pa{Z}}{\cH_{t_{j,q}}}},\]
where $t_{k,q}$ is the round $t$ where $\fS_t\pa{Z}$ holds for the $k$-th time since the beginning of the round $\tau_q+1$. Within the expectation, one can recognize the expectation of a time-varying geometric distribution, where the success probability of the $k$-th trial is  $\PPc{\neg\fT_{t_{k,q}}\pa{Z}}{\cH_{t_{k,q}}}$. We can upper bound this inner expectation by the expectation of a geometric distribution whose success probability
\[\inf_{\tau\geq{\tau_q+1}}\PPc{\neg\fT_{\tau}\pa{Z}}{\cH_{\tau}}=\inf_{\tau\geq{\tau_q+1}}\prod_{i\in Z}\PPc{\abs{\theta_{i,\tau}-\mu_i^*}\leq \eps}{\cH_{\tau}}\] is lower than all the success probabilities of the time-varying geometric distribution. This gives the result by monotonicity of the expectation, and rewriting the expectation of the geometric distribution.
\end{proof}

\begin{proof}[Proof of Lemma~\ref{lem:Esupbound}]
For any arm $i\in [n]$, $k_i\in \N$, we define
$p_{i,k_i}$ as the probability of $\abs{\tilde\theta_{i,k_i}-\mu_i^*}\leq \eps$, where $\tilde\theta_{i,k_i}$ is a sample from the posterior of arm $i$ when there are $k_i$ observations of arm $i$ (i.e., $p_{i,k_i}$ is a random variable measurable with respect to those $k_i$ independent draws of arm $i$).
From Lemma~5,6 in \citet{Wang2018}, we know that  
\[{\EE{\frac{1}{p_{i,k_i}}}}\leq \left\{
    \begin{array}{ll}
        4/\eps^2 & \mbox{for every } k_i\geq 0 \\
        1+6c''\cdot{e^{-\eps^2 k_i/2}}\eps^{-2} +\frac{2}{e^{\eps^2k_i/8}-2}& \mbox{if }  k_i> 8/\eps^2,
    \end{array}
\right.\]
for some universal constant $c''$.
Since $\fS_t\pa{Z}\wedge\neg\fT_t\pa{Z}$ implies that $Z\subset A_t$, we know that 
for $\tau\geq \tau_q+1$, $N_{i,\tau-1}\geq q$ for all $i\in Z$. Using the mutual independence of outcomes, and the fact that the distribution of $\theta_{i,\tau}$ depends only on the history of arm $i$, we have 
\begin{align*}
&\EE{\sup_{\tau\geq{\tau_q+1}}\prod_{i\in Z}\frac{1}{\PPc{\abs{\theta_{i,\tau}-\mu_i^*}\leq \eps}{\cH_\tau}}}-1
    \\&=
    \EE{\sup_{\tau\geq{\tau_q+1}}\sum_{Z'\subset Z,~Z'\neq \emptyset}{\prod_{i\in Z'}\pa{\frac{1}{\PPc{\abs{\theta_{i,\tau}-\mu_i^*}\leq \eps}{\cH_\tau}}-1}}}
    \\&\leq
    \sum_{Z'\subset Z,~Z'\neq \emptyset}\EE{\prod_{i\in Z'}\sup_{\tau\geq{\tau_q+1}}\pa{\frac{1}{\PPc{\abs{\theta_{i,\tau}-\mu_i^*}\leq \eps}{\cH_\tau}}-1}}
      \\&\leq
     \sum_{Z'\subset Z,~Z'\neq \emptyset}\EE{\prod_{i\in Z'}\sum_{k_i\geq q}\pa{\frac{1}{p_{i,k_i}}-1}},
     \\&=
     \sum_{Z'\subset Z,~Z'\neq \emptyset}\prod_{i\in Z'}\EE{\sum_{k_i\geq q}\pa{\frac{1}{p_{i,k_i}}-1}}.
     \end{align*}
     From this point, there are two cases:
     If $q>8/\eps^2$,
     \begin{align*}
     &\leq
     \sum_{Z'\subset Z,~Z'\neq \emptyset}\prod_{i\in Z'}\sum_{k_i\geq q}\pa{6c''\cdot{e^{-\eps^2 k/2}}\eps^{-2} +{2e^{-\eps^2k/8}}\pa{1-2e^{-\eps^2k/8}}^{-1}}
     \\&\leq
     \sum_{Z'\subset Z,~Z'\neq \emptyset}\prod_{i\in Z'}\pa{6c''\cdot{e^{-\eps^2 q/2}}\eps^{-2}\sum_{k\geq 0}e^{-\eps^2 k/2} +2e^{-\eps^2q/8}\pa{1-2e^{-\eps^2q/8}}^{-1}\sum_{k\geq 0}e^{-\eps^2k/8}}
       \\&=
     \sum_{Z'\subset Z,~Z'\neq \emptyset}\prod_{i\in Z'}\pa{6c''\cdot{e^{-\eps^2 q/2}}\eps^{-2}\pa{1-e^{-\eps^2/2}}^{-1} +2e^{-\eps^2q/8}\pa{1-2e^{-\eps^2q/8}}^{-1}\pa{1-e^{-\eps^2/8}}^{-1}}
     \\&\leq
         \sum_{Z'\subset Z,~Z'\neq \emptyset}\prod_{i\in Z'}\pa{6c''\cdot{e^{-\eps^2 q/2}}\eps^{-2}\cdot2\eps^{-2}\pa{1-e^{-1/2}}^{-1} +2e^{-\eps^2q/8}\pa{1-2e^{-1}}^{-1}\cdot 8\eps^{-2}\pa{1-e^{-1/8}}^{-1}}
    \\&\leq \sum_{Z'\subset Z,~Z'\neq \emptyset} e^{-\abs{Z'}\eps^2 q/8}\pa{12c''\cdot{e}^{-3}\pa{1-e^{-1/2}}^{-1}\cdot\eps^{-4} + 16\pa{1-2e^{-1}}^{-1}\eps^{-2}\pa{1-e^{-1/8}}^{-1}}^{\abs{Z'}}
    \\&\leq  e^{-\eps^2 q/8}\pa{12c''\cdot{e}^{-3}\pa{1-e^{-1/2}}^{-1}\eps^{-4} + 16\pa{1-2e^{-1}}^{-1}\eps^{-2}\pa{1-e^{-1/8}}^{-1}+1}^{\abs{Z}}
    \\&\leq  e^{-\eps^2 q/8}\pa{c'\eps^{-4}}^{\abs{Z}} ,
\end{align*}
and if $q\leq 8/\eps^2$,
\begin{align*}
    &\leq
     \sum_{Z'\subset Z,~Z'\neq \emptyset}\prod_{i\in Z'}\pa{\sum_{k= q}^{\floor{8/\eps^2}}\pa{4/\eps^2-1}+\sum_{k\geq \floor{8/\eps^2}+1}^{\infty}\pa{6c\cdot{e^{-\eps^2 k/2}}\eps^{-2} +{2e^{-\eps^2k/8}}\pa{1-2e^{-\eps^2k/8}}^{-1}}}
     \\&\leq \sum_{Z'\subset Z,~Z'\neq \emptyset}\prod_{i\in Z'}\pa{36 \eps^{-4}+12c\cdot{e}^{-4}\pa{1-e^{-1/2}}^{-1}\eps^{-4} + 16e^{-1}\pa{1-2e^{-1}}^{-1}\eps^{-2}\pa{1-e^{-1/8}}^{-1}}
     \\&\leq 
     \pa{  c\eps^{-4}}^{\abs{Z}} ,
\end{align*}
where $c,c'$ are two universal constant.
\end{proof}

\subsection{Discussion on the new exponential constant term (step 4 in the above proof)}
 

 We give here an explanation concerning the modification of Lemma~7 from \cite{Wang2018}.
 First, we respectfully disagree with the end of their proof, where
 the expected number of time slots for $\fS_t\pa{Z}\wedge\neg\fT_t\pa{Z}$ to occur is a weighted mean of expectations where the counters are fixed and non-random. 
 To obtain such a weighted mean, they
 have conditioned on the value of the counters. However, 
 counters depend on the chosen action, and thus on the outcomes previously obtained, so conditioning on it would modify the expectation, since the term inside the expectation not only depends on counters, but also on outcomes obtained so far.
To illustrate more clearly this point, let us focus on one arm $i$, and consider the extreme case where we get a new sample (i.e. the counter is incremented) only if samples $Y_{i,t}$ previously obtained from $i$ were all $0$, say. Then conditioning on the fact that the counter is incremented would remove all the randomness of samples $Y_{i,t}$, and we thus can't consider an expectation on those samples as if their randomness was not impacted. 

We now expose our approach  to overcome this issue.  We first rewrite the above mentioned expected number of time slots as the expectation (over the history) of the expectation of a time-varying geometric distribution, where the time-varying success probability depends on the history. 
 The inner expectation can be bounded by the expectation of a geometric distribution whose success probability is the infimum over all the success probabilities of the time-varying geometric distribution. 
 Let's note that this gives us the inverse success probability minus one, as in \citet{Wang2018}, but that counters are still random.
 We use that this inverse probability can be factorized: from the relation $\prod_{i\in A}a_i-1=\sum_{A'\subset A,~A'\neq \emptyset}\prod_{i\in A'}\pa{a_i-1}$, valid for any vector $\ba=(a_i)$ on a set $A$, and from the mutual independence of outcomes, we're reduced to bounding the expectation in the one-dimensional case.
To overcome the randomness of the counters, we use an union bound.
It is this union bound that brings a larger dependence on the constant term, because it forces us to look at a sum of the form $\sum_q\sum_{k\geq q} x_k$, instead of a simply $\sum_q x_q$. 
Let's remark that \citet{Wang2018} use the eventual exponential decreasing of the sequence $\pa{x_q}$ in order to get their final bound.
 We manage to deal with the sequence $\pa{\sum_{k\geq q} x_k}$ instead, by noticing that the eventual exponential decreasing of the sequence $\pa{x_q}$ implies the eventual exponential decreasing of the sequence $\pa{\sum_{k\geq q} x_k}$.

\section{Proof of Proposition~\ref{prop:subgau}}\label{app:subgau}
 Assumption~\ref{ass:subgau} encompasses 
 $\kappa_i^2$-sub Gaussian outcomes with $D_i=\kappa_i^2 m$ for all $i\in [n]$. Indeed, let $\blambda=\blambda\odot\be_A$ for some action $A$ and observe that 
  \[\EE{e^{\blambda\transpose\pa{\bX-\bmu^*}}}\leq\EE{\sum_i \frac{\abs{\kappa_i\lambda_i}}{\norm{\bkappa\odot\blambda}_1}e^{\norm{\bkappa\odot\blambda}_1\sign\pa{\lambda_i}\frac{X_i-\mu_i^*}{\kappa_i}}}\leq e^{\norm{\bkappa\odot\blambda}_1^2/2}\leq e^{\norm{\bkappa\odot\blambda}^2_2\abs{A}/2}\leq e^{\norm{\bkappa\odot\blambda}^2_2m/2}. \]
  The case of $\bC$-sub-Gaussian outcomes with a known sub-Gaussian matrix $\bC$ (i.e., $\EE{e^{\blambda\transpose\pa{\bX-\bmu^*}}}\leq e^{\blambda\transpose\bC\blambda/2}$ for all $\blambda\in \R^n$) is also captured, taking\footnote{This $D_i$ can be computed whenever linear maximization on $\cA$ is efficient: for $x$ high enough, we have
  \(\max_{A\in \cA,~i\in A}\sum_{j\in A}{\abs{C_{ij}}}={C_{ii}}-x+\max_{A\in \cA}\sum_{j\in A}\pa{\abs{C_{ij}}\II{j\neq i}+x\II{j= i}}.\) 
  }  $D_i=\max_{A\in \cA,~i\in A}\sum_{j\in A}\abs{C_{ij}}$. Indeed, for an action $A$,  \[\sum_{i,j\in A}\lambda_i\lambda_jC_{ij}\leq \sum_{i,j\in A}\frac{\lambda_i^2+\lambda_j^2}{2}\abs{C_{ij}}=\sum_{i\in A}\lambda_i^2\sum_{j\in A}\abs{C_{ij}}\leq \sum_{i\in n}\lambda_i^2\max_{A\in \cA,~i\in A}\sum_{j\in A}{\abs{C_{ij}}}.\]

\section{Proof of Theorem~\ref{thm:tsgauss}}
\label{app:tsgauss}
We beginning by stating the complete version of Theorem~\ref{thm:tsgauss}.
\begin{theorem*}
  The policy $\pi$ described in
 Algorithm~\ref{algo:tsgauss} has regret $R_T(\pi)$ bounded by
\begin{align*}
 &256\log_2^2(4\sqrt{m})\sum_{i\in [n]}\frac{{B^2 \beta D_i\log\pa{2^m\abs{\cA}T}}}{\Delta_{i,\min}} +
\Delta_{\max}(1+2n)\\+&\frac{nm^2\Delta_{\max}}{\pa{\frac{\Delta_{\min}}{2B}-({m^*}^2+1)\eps}^{2}}+\Delta_{\max}\pa{C\eps^{-2}\beta\max_i D_i}\pa{\frac{C'}{\sqrt{\beta-1}}\eps^{-4}\beta^3\max_i D_i^2}^{m^*},
\end{align*}
where $C,C'$ are two universal constants, and $\eps\in (0,1)$ is such that ${\Delta_{\min}}/(2B)-({m^*}^2+1)\eps>~0.$
 \end{theorem*}

For the proof of Theorem~\ref{thm:tsgauss}, we consider the same events as in the proof of Theorem~\ref{thm:tsbeta}, except for the event $\fD_t$, that becomes
\[\fD_t\triangleq \set{\norm{\be_{A_t}\odot\pa{\btheta_{t}-\bar\bmu_{t-1}} }_1\geq\sqrt{{2}{\log\pa{\abs{\cA}2^mT}}\sum_{i\in {A_t}}{\beta D_i}/{N_{i,t-1}}}}.\]
Step 1 is unchanged. Step 2 and Step 3 are modified only through the event $\fD_t$, using the following modification of Lemma~\ref{lem:nlem4}.
\begin{lemma}In Algorithm~\ref{algo:tsgauss}, for all round $t$, we have
that $\PPc{\fD_t}{\cH_t}\leq 1/T$.
\end{lemma}
\begin{proof}
We rely on the fact that conditionally on the history, the sample $\btheta_t$ is Gaussian of mean $\bar\bmu_{t-1}$ and of diagonal covariance given by $\beta D_iN_{i,t-1}^{-1}$. We thus define the functions \[\alpha_t(A)\triangleq\sqrt{2{\log\pa{\abs{\cA}2^mT}}\sum_{i\in {A}}\frac{\beta D_i}{N_{i,t-1}}},\quad\text{and}\quad \lambda_t(A) \triangleq \frac{\alpha_t(A)}{\sum_{i\in A}\beta D_i/N_{i,t-1}}, \] we have 
\begin{align*}\PPc{\norm{\be_{A_t}\odot\pa{\btheta_{t}-\bar\bmu_{t-1}} }_1\geq \alpha_t(A_t)}{\cH_t}&\leq\sum_{A\in \cA}\PPc{\norm{\be_{A}\odot\pa{\btheta_{t}-\bar\bmu_{t-1}} }_1\geq \alpha_t(A)}{\cH_t}\\&\leq \sum_{A\in \cA} e^{-\lambda_t(A) \alpha_t(A)}\EEc{e^{\lambda_t(A) \norm{\be_{A}\odot\pa{\btheta_{t}-\bar\bmu_{t-1}} }_1}}{\cH_t}
\\&\leq 
\sum_{A\in \cA} e^{-\lambda_t(A) \alpha_t(A)}\prod_{i\in A}\EEc{e^{\lambda_t(A) \abs{{\theta_{i,t}-\bar\mu_{i,t-1}} }}}{\cH_t}
\\&\leq \sum_{A\in \cA} e^{-\lambda_t(A) \alpha_t(A)}\prod_{i\in A}\EEc{e^{\lambda_t(A) \pa{{\theta_{i,t}-\bar\mu_{i,t-1}} }}+e^{\lambda_t(A) \pa{{\bar\mu_{i,t-1} -\theta_{i,t}} }}}{\cH_t}
\\&\leq \sum_{A\in \cA} 2^{\abs{A}}e^{-\lambda_t(A) \alpha_t(A)}{e^{\lambda_t(A)^2 {\sum_{i\in A}\beta D_i/(2 N_{i,t-1}) }}}\leq 1/T.
\end{align*}
\end{proof}
The final bound on the regret in Step 3 is 
 obtained using the same derivation as in Theorem~\ref{thm:tsbeta}, which gives the following leading term:
 \[256\log_2^2(4\sqrt{m})\sum_{i\in [n]}\frac{{B^2 \beta D_i\log\pa{2^m\abs{\cA}T}}}{\Delta_{i,\min}}.\]
In the following, we consider the last step, consisting in bounding the regret under the event $\fZ_t$ and $\neg\fC_t$.
From the initialization phase, we also assume that the event 
\[\fM_t\triangleq \set{\forall i\in [n],~N_{i,t-1}\geq 1}\]
holds (the regret under the complementary event is clearly bounded by $n\Delta_{\max}$). If there is no initialization, we can have $q=0$ in the following, noticing that when $\theta_{i,t}$ is uniform on $[a,b]$, then the probability $\PPc{\abs{\theta_{i,t}-\mu_i^*}\leq \eps}{\cH_t}$ is equal to $2\eps/(b-a)$.
\paragraph{Step 4: bound under $\fM_t\wedge\fZ_t\wedge\neg\fC_t$}
We use the independence of the prior, as for Theorem~\ref{thm:tsbeta}, to obtain the following upper bound, using $\fM_t$ to be able to start from $q=1$. 
\begin{align*}\sum_{t\in [T]}\EE{\Delta\pa{A_t}\II{\fM_t\wedge\fZ_t\wedge\neg\fC_t}}&\leq
    \sum_{Z\subset A^*,~Z\neq \emptyset}\sum_{q\geq 1}\EE{\sup_{\tau\geq{\tau_q+1}}\sum_{Z'\subset Z,~Z'\neq \emptyset}{\prod_{i\in Z'}\pa{\frac{1}{\PPc{\abs{\theta_{i,\tau}-\mu_i^*}\leq \eps}{\cH_\tau}}-1}}}\\&\leq
   \sum_{Z\subset A^*,~Z\neq \emptyset}\sum_{q\geq 1} \underbrace{\sum_{Z'\subset Z,~Z'\neq \emptyset} \EE{\sup_{\tau\geq{\tau_q+1}}{\prod_{i\in Z'}\pa{\frac{1}{\PPc{\abs{\theta_{i,\tau}-\mu_i^*}\leq \eps}{\cH_\tau}}-1}}}}_{\numterm{rel:tsterm1thm2}}
    .
\end{align*}
However, the expectation can't be put inside the product since outcomes are not mutually independent.
We can still take a union bound on counters: 
\begin{align*}
   \eqref{rel:tsterm1thm2}\leq\sum_{Z'\subset Z,~Z'\neq \emptyset}~
   \sum_{\bk\in [q..\infty)^{Z'}}
   \EE{\sup_{\tau\geq{\tau_q+1}}\II{\forall i \in Z',~N_{i,\tau-1}=k_i}{\prod_{i\in Z'}\pa{\frac{1}{\PPc{\abs{\theta_{i,\tau}-\mu_i^*}\leq \eps}{\cH_\tau}}-1}}}.
\end{align*}
One can notice that for all $i\in Z'$, all $k_i\geq q$, $\II{N_{i,\tau-1}=k_i}\pa{\frac{1}{\PPc{\abs{\theta_{i,\tau}-\mu_i^*}\leq \eps}{\cH_\tau}}-1}$ is of the form  $\II{N_{i,\tau-1}=k_i}g_i\pa{\abs{\bar\mu_{i,\tau-1}-\mu_i^*}}$, with $g_i$ being an increasing function on $\R_+$. Indeed, we see that the conditional distribution of $\theta_{i,\tau}-\bar\mu_{i,\tau-1}$ is $\cN\pa{0,\beta D_iN_{i,\tau-1}^{-1}}$, which is symmetric, so we have 
\[\PPc{\abs{\theta_{i,\tau}-\mu_i^*}\leq \eps}{\cH_\tau}=\PPc{\abs{\theta_{i,\tau}-\bar\mu_{i,\tau-1}+\abs{\bar\mu_{i,\tau-1}-\mu_i^*}}\leq \eps}{\cH_\tau}.\]
In addition, under $\II{N_{i,\tau-1}=k_i}$, the conditional distribution of $\theta_{i,\tau}-\bar\mu_{i,\tau-1}$ does not depend on the history, but only on $k_i$. Therefore, the above probability is a function of $\abs{\bar\mu_{i,\tau-1}-\mu_i^*}$ and so the function $g_i$ exists. It
is increasing on $\R_+$ because for any fixed $\sigma>0$,
\[\frac{\partial}{\partial x}\int_{x-\eps}^{x+\eps}\frac{1}{\sqrt{2\pi\sigma^2}}e^{-\frac{u^2}{2\sigma^2}}\mathrm{d}u=
\frac{1}{\sqrt{2\pi\sigma^2}}\pa{e^{-\frac{(x+\eps)^2}{2\sigma^2}}-e^{-\frac{(x-\eps)^2}{2\sigma^2}}}< 0\text{ for }x> 0.
\]
In particular, we can consider the inverse function $g_i^{-1}$.
We now want to use a stochastic dominance argument in order to treat the outcomes as if they were Gaussian: we have for any $\bk\in [q..\infty)^{Z'}$,
\begin{align}\nonumber
    &\EE{\sup_{\tau\geq{\tau_q+1}}\prod_{i\in Z'}\pa{\II{N_{i,\tau-1}=k_i}g_i\pa{\abs{\bar\mu_{i,\tau-1}-\mu_i^*}}}}\\&=
    \EE{\sup_{\tau\geq{\tau_q+1}}\prod_{i\in Z'}\pa{\II{N_{i,\tau-1}=k_i}\int_{0}^\infty \II{g_i\pa{\abs{\bar\mu_{i,\tau-1}-\mu_i^*}}\geq u_i}\mathrm{d}u_i}}\nonumber\\&\leq\int_{\bu\in \R_+^{Z'}}\EE{\sup_{\tau\geq{\tau_q+1}}\prod_{i\in Z'}{\II{N_{i,\tau-1}=k_i} \II{g_i\pa{\abs{\bar\mu_{i,\tau-1}-\mu_i^*}}\geq u_i}}} \mathrm{d}\bu
    \nonumber\\&=
    \int_{\bu\in \R_+^{Z'}}\EE{\prod_{i\in Z'}{ \II{N_{i,\tau^*-1}=k_i}\II{g_i\pa{\abs{\bar\mu_{i,\tau^*-1}-\mu_i^*}}\geq u_i}}} \mathrm{d}\bu
    \label{rel:stodom},
\end{align}
where $\tau^*$ is the first time $\tau$ such that the event $\II{\forall i \in Z',~N_{i,\tau-1}=k_i~\text{and}~g_i\pa{\abs{\bar\mu_{i,\tau-1}-\mu_i^*}}\geq u_i}$ holds, and is $\infty$ if it never holds. 
\begin{align*}
    \eqref{rel:stodom}&= \int_{\bu\in \R_+^{Z'}}\EE{\prod_{i\in Z'}{ \II{N_{i,\tau^*-1}=k_i}\II{g_i\pa{\abs{\bar\mu_{i,\tau^*-1}-\mu_i^*}}\geq u_i\vee g_i(0)}}} \mathrm{d}\bu
    \\&= \int_{\bu\in \R_+^{Z'}}\EE{\prod_{i\in Z'}{ \II{N_{i,\tau^*-1}=k_i}\II{{\abs{\bar\mu_{i,\tau^*-1}-\mu_i^*}}\geq g_i^{-1}\pa{u_i\vee g_i(0)}}}} \mathrm{d}\bu
    \\&= \int_{\bu\in \R_+^{Z'}}\sum_{\bs \in \set{-1,1}^{Z'}}\underbrace{\EE{\prod_{i\in Z'}{ \II{N_{i,\tau^*-1}=k_i}\II{{s_i\pa{\bar\mu_{i,\tau^*-1}-\mu_i^*}}\geq g_i^{-1}\pa{u_i\vee g_i(0)}}}}}_{\numterm{rel:subgauetstodom}}\mathrm{d}\bu
    \end{align*}
    
    \begin{align*}
\eqref{rel:subgauetstodom}    
    &\leq {\PP{\frac{\exp\pa{\!\sum_{i\in Z'}{ \!N_{i,\tau^*-1}\!  \pa{\frac{s_ig_i^{-1}\pa{u_i\vee g_i(0)}}{D_i}\pa{\bar\mu_{i,\tau^*-1}\!-\!\mu_i^*}-\frac{\pa{g_i^{-1}\pa{u_i\vee g_i(0)}}^2}{2D_i}}}}}{\exp\pa{\sum_{i\in Z'}\frac{\pa{g_i^{-1}\pa{u_i\vee g_i(0)}}^2k_i}{2D_i}}}\geq 1,\pa{N_{i,\tau^*-1}}_{i\in Z'}\!=\!\bk}} 
    \\&\leq
    {\PP{\frac{\exp\pa{\sum_{i\in Z'}{ N_{i,\tau^*-1}  \pa{\frac{s_ig_i^{-1}\pa{u_i\vee g_i(0)}}{D_i}\pa{\bar\mu_{i,\tau^*-1}-\mu_i^*}-\frac{\pa{g_i^{-1}\pa{u_i\vee g_i(0)}}^2}{2D_i}}}}}{\exp\pa{\sum_{i\in Z'}\frac{\pa{g_i^{-1}\pa{u_i\vee g_i(0)}}^2k_i}{2D_i}}}\geq 1}} 
    \\&\leq
    {\frac{\EE{\exp\pa{\sum_{i\in Z'}{ N_{i,\tau^*-1}  \pa{\frac{s_ig_i^{-1}\pa{u_i\vee g_i(0)}}{D_i}\pa{\bar\mu_{i,\tau^*-1}-\mu_i^*}-\frac{\pa{g_i^{-1}\pa{u_i\vee g_i(0)}}^2}{2D_i}}}}}}{\exp\pa{\sum_{i\in Z'}\frac{\pa{g_i^{-1}\pa{u_i\vee g_i(0)}}^2k_i}{2D_i}}}}
    \\&=
    {\frac{\EE{\exp\pa{\sum_{t=1}^{\tau^*-1}\sum_{i\in Z'\cap A_t}{   \pa{\frac{s_ig_i^{-1}\pa{u_i\vee g_i(0)}}{D_i}\pa{X_{i,t}-\mu_i^*}-\frac{\pa{g_i^{-1}\pa{u_i\vee g_i(0)}}^2}{2D_i}}}}}}{\exp\pa{\sum_{i\in Z'}\frac{\pa{g_i^{-1}\pa{u_i\vee g_i(0)}}^2k_i}{2D_i}}}}.
\end{align*}
From Assumption~\ref{ass:subgau}, we have that  \[M_{\tau}=\exp\pa{\sum_{t=1}^{\tau-1}\sum_{i\in Z'\cap A_t}{   \pa{\frac{s_ig_i^{-1}\pa{u_i\vee g_i(0)}}{D_i}\pa{X_{i,t}-\mu_i^*}-\frac{\pa{g_i^{-1}\pa{u_i\vee g_i(0)}}^2}{2D_i}}}}\]
is a supermartingale:
\begin{align*}\EEc{M_{\tau}}{\cF_{\tau-1}}&=M_{\tau-1}\EEc{\exp{\pa{\sum_{i\in Z'\cap A_{\tau-1}}{   \pa{\frac{s_ig_i^{-1}\pa{u_i\vee g_i(0)}}{D_i}\pa{X_{i,{\tau-1}}-\mu_i^*}-\frac{\pa{g_i^{-1}\pa{u_i\vee g_i(0)}}^2}{2D_i}}}}}}{\cF_{\tau-1}}\\&\leq M_{\tau-1}.\end{align*}
Since $\tau^*$ is a stopping time with respect to $\cF_\tau$, we have from Doob's optional sampling theorem for non-negative supermartingales\footnote{We use the version that relies on Fatou's lemma (\citet{durrett2019probability}, Theorem 5.7.6), so that it is not needed to have any additional condition on the stopping time $\tau^*$.} that $\EE{M_{\tau^*}}\leq 1$. Therefore, 
\begin{align*}
\eqref{rel:subgauetstodom}    
    &\leq \exp\pa{-\sum_{i\in Z'}\frac{\pa{g_i^{-1}\pa{u_i\vee g_i(0)}}^2k_i}{2D_i}}.
    \end{align*}
    Now, we want to use the following fact (see \citet{chang2011chernoff}):
    if $\eta\sim\cN(0,1)$, then with $\beta>1$,
    \[\sqrt{\frac{2e}{\pi}}\frac{\sqrt{\beta-1}}{\beta}e^{-\beta x^2/2}\leq \PP{\abs{\eta}\geq x}.\]
    Indeed, this gives
    \[\sqrt{\frac{2e}{\pi}}\frac{\sqrt{\beta-1}}{\beta}\exp\pa{-\frac{\pa{g_i^{-1}\pa{u_i\vee g_i(0)}}^2k_i}{2D_i}}\leq \PP{\abs{\eta_i}\geq  {g_i^{-1}\pa{u_i\vee g_i(0)}}\sqrt{\frac{k_i}{\beta D_i}}},\]
    where $\bheta\sim\cN(0,1)^{\otimes Z'}$.
    Thus, \begin{align*}
     \eqref{rel:stodom}&\leq \pa{\sqrt{\frac{\pi}{2e}}\frac{2\beta}{\sqrt{\beta-1}}}^{\abs{Z'}} \int_{\bu\in \R_+^{Z'}}\prod_{i\in Z'}\PP{\sqrt{\frac{\beta D_i}{k_i}}\abs{\eta_i}\geq  {g_i^{-1}\pa{u_i\vee g_i(0)}}}\mathrm{d}\bu\\&
     = \pa{\sqrt{\frac{\pi}{2e}}\frac{2\beta}{\sqrt{\beta-1}}}^{\abs{Z'}} \int_{\bu\in \R_+^{Z'}}\prod_{i\in Z'}\PP{g_i\pa{\sqrt{\frac{\beta D_i}{k_i}}\abs{\eta_i}}\geq  {{u_i\vee g_i(0)}}}\mathrm{d}\bu\\&
     = \pa{\sqrt{\frac{\pi}{2e}}\frac{2\beta}{\sqrt{\beta-1}}}^{\abs{Z'}} \int_{\bu\in \R_+^{Z'}}\prod_{i\in Z'}\PP{g_i\pa{\sqrt{\frac{\beta D_i}{k_i}}\abs{\eta_i}}\geq  {{u_i}}}\mathrm{d}\bu\\&
     = \pa{\sqrt{\frac{\pi}{2e}}\frac{2\beta}{\sqrt{\beta-1}}}^{\abs{Z'}} \prod_{i\in Z'}\int_{0}^\infty\PP{g_i\pa{\sqrt{\frac{\beta D_i}{k_i}}\abs{\eta_i}}\geq  {{u_i}}}\mathrm{d}u_i\\&
     = \pa{\sqrt{\frac{\pi}{2e}}\frac{2\beta}{\sqrt{\beta-1}}}^{\abs{Z'}} \prod_{i\in Z'}\EE{g_i\pa{\sqrt{\frac{\beta D_i}{k_i}}\abs{\eta_i}}}.
    \end{align*}
    We now want to bound $\EE{g_i\pa{\sqrt{\frac{\beta D_i}{k_i}}\abs{\eta_i}}}.$ We define $\alpha =2 - \sqrt{2}$, the unique solution in $(1/2,1)$ of $\alpha-1/2=(\alpha-1)^2/2$. Notice that $\alpha-1/2\geq 1/12$. Define $\eps_i\triangleq\eps\sqrt{\frac{k_i}{\beta D_i}}$.
    By definition, we have
\begin{align*}
    \EE{g_i\pa{\sqrt{\frac{\beta D_i}{k_i}}\abs{\eta_i}}} & = \int_{-\infty}^{+\infty} \frac{e^{-x^2/2}}{\int_{x - \varepsilon_i}^{x + \varepsilon_i} e^{-y^2/2} \mathrm{d}y}\mathrm{d}x - 1 \\
    & =  \underbrace{2\int_{\alpha \varepsilon_i}^{+ \infty}\frac{1}{\int_{x - \varepsilon_i}^{x + \varepsilon_i} e^{-\frac{y^2-x^2}{2}} \mathrm{d}y} \mathrm{d}x}_{A_1} + \underbrace{\int_{-\alpha \varepsilon_i}^{\alpha \varepsilon_i}\frac{e^{-x^2/2}}{\int_{x - \varepsilon_i}^{x + \varepsilon_i} e^{-y^2/2} \mathrm{d}y}\mathrm{d}x - 1}_{A_2}.
\end{align*}
We first bound $A_1$. With the change of variable $u = y-x$, we get:
\begin{align*}
    A_1 & = 2\int_{\alpha \varepsilon_i}^{+ \infty}\frac{1}{\int_{- \varepsilon_i}^{\varepsilon_i} e^{-u^2/2 - ux} \mathrm{d}u} \mathrm{d}x \\
    & \leq 2\int_{\alpha \varepsilon_i}^{+ \infty}\frac{1}{\int_{-\varepsilon_i}^{0} e^{-u^2/2 - ux} \mathrm{d}u} \mathrm{d}x
\end{align*}
Note that for $x\geq \alpha \varepsilon_i$ and $u \in [-\eps_i, 0]$, $-u^2/2 - ux \geq -(1-\frac{1}{2\alpha})ux$ and thus:
\begin{align}
    A_1 & \leq 2\int_{\alpha \varepsilon_i}^{+ \infty}\frac{1}{\int_{-\varepsilon_i}^{0} e^{-(1-\frac{1}{2\alpha})ux} \mathrm{d}u} \mathrm{d}x \nonumber\\
    & = 2\int_{\alpha \varepsilon_i}^{+ \infty} \frac{(1-\frac{1}{2\alpha})x}{e^{(1-\frac{1}{2\alpha})\varepsilon_i x}-1} \mathrm{d}x.\label{rel:tworegimes} \end{align}
    We distinguish two regimes. First, if $\eps_i^2\geq 12$, then
    \begin{align*}\eqref{rel:tworegimes}
    &\leq \frac{2e^{\pa{\alpha-\frac{1}{2}}\varepsilon_i^2}}{e^{(\alpha-\frac{1}{2})\varepsilon_i^2}-1}\int_{\alpha \varepsilon_i}^{+ \infty} \pa{1-\frac{1}{2\alpha}}xe^{-(1-\frac{1}{2\alpha})\varepsilon_i x} \mathrm{d}x \\
    & = \frac{2e^{(\alpha-\frac{1}{2})\varepsilon_i^2}}{e^{(\alpha-\frac{1}{2})\varepsilon_i^2}-1}\frac{1}{(1-\frac{1}{2\alpha})\varepsilon_i^2}\int_{(\alpha-\frac{1}{2})\varepsilon_i^2}^{+ \infty}xe^{- x} \mathrm{d}x \\
    & = \frac{2e^{(\alpha-\frac{1}{2})\varepsilon_i^2}}{e^{(\alpha-\frac{1}{2})\varepsilon_i^2}-1} \frac{1}{(1-\frac{1}{2\alpha})\varepsilon_i^2}\left[-(x+1)e^{-x} \right]_{(\alpha-\frac{1}{2})\varepsilon_i^2}^{\infty} \\
    & = \frac{2e^{(\alpha-\frac{1}{2})\varepsilon_i^2}}{e^{(\alpha-\frac{1}{2})\varepsilon_i^2}-1} \frac{1}{(1-\frac{1}{2\alpha})\varepsilon_i^2} \pa{\pa{\alpha-\frac{1}{2}}\varepsilon_i^2+1}e^{-(\alpha-\frac{1}{2})\varepsilon_i^2} \\
    & = \frac{2}{e^{(\alpha-\frac{1}{2})\varepsilon_i^2}-1}\pa{ \alpha+\frac{\alpha}{(\alpha-\frac{1}{2})\varepsilon_i^2}}
    \\
    & \leq {4 e^{-\varepsilon_i^2/12}} .
\end{align*}
Otherwise, we have
\begin{align*}
    \eqref{rel:tworegimes} &= \frac{2(1-\frac{1}{2\alpha})}{\varepsilon_i^2}\int_{\alpha \varepsilon_i^2}^\infty \frac{u}{e^{\pa{1-\frac{1}{2\alpha}}u}-1} \mathrm{d}u \\&\leq \frac{2(1-\frac{1}{2\alpha})}{\varepsilon_i^2}\int_{0}^\infty \frac{u}{e^{\pa{1-\frac{1}{2\alpha}}u}-1} \mathrm{d}u \\&= \frac{2(1-\frac{1}{2\alpha})}{\varepsilon_i^2}\frac{\pi^2}{6\pa{1-\frac{1}{2\alpha}}^2}
    \\&\leq \frac{24\beta D_i}{\eps^2}.
\end{align*}

We now bound $A_2$. 
As $x \in [-\alpha \varepsilon_i, \alpha \varepsilon_i]$, it comes that $[-(1-\alpha)\varepsilon_i, (1-\alpha)\varepsilon_i] \subset [x-\varepsilon_i, x+\varepsilon_i]$. This implies that
\begin{align*}
A_2 & \leq \frac{\int_{-\alpha \varepsilon_i}^{\alpha \varepsilon_i} e^{-x^2/2} \mathrm{d} x }{\int_{-(1-\alpha) \varepsilon_i}^{(1-\alpha) \varepsilon_i} e^{-x^2/2} \mathrm{d} x} - 1\\
& = \frac{2 \int_{(1-\alpha)\varepsilon_i}^{\alpha \varepsilon_i} e^{-x^2/2}\mathrm{d}x}{\int_{-(1-\alpha) \varepsilon_i}^{(1-\alpha) \varepsilon_i} e^{-x^2/2} \mathrm{d} x}  \\
& \leq \frac{2 \int_{(1-\alpha)\varepsilon_i}^{\infty} e^{-x^2/2}\mathrm{d}x}{\int_{-(1-\alpha) \varepsilon_i}^{(1-\alpha) \varepsilon_i} e^{-x^2/2} \mathrm{d} x} 
\\
& \leq \frac{  e^{-{(1-\alpha)^2\varepsilon_i^2}/2}}{1-e^{-{(1-\alpha)^2\eps_i^2}/{2}}}\leq \pa{1+\frac{12}{\eps_i^2}}e^{-\eps_i^2/12}.
\end{align*}
The penultimate inequality relies on  $\int_{x}^\infty e^{-u^2/2} \mathrm{d}u {\leq} \sqrt{\frac{\pi}{2}}{e^{-x^2/2}}$ (see \citet{jacobs1965principles}, eq. (2.122)).
We obtain again two regimes: $2e^{-\eps_i^2/12}$ if $\eps_i^2\geq 12$, and $1+\frac{12\beta D_i}{\eps^2}$ otherwise.
To summarize, we proved that
\[\eqref{rel:stodom}\leq  \pa{\sqrt{\frac{\pi}{2e}}\frac{2\beta}{\sqrt{\beta-1}}}^{\abs{Z'}} \prod_{i\in Z'} \pa{\II{\eps^2{\frac{k_i}{\beta D_i}}< 12}\pa{1+36\frac{\beta D_i}{\eps^2}}
+\II{\eps^2{\frac{k_i}{\beta D_i}}\geq 12}{6e^{-\eps^2{\frac{k_i}{12\beta D_i}}}}}.\]
After the summation on $\bk$, on $Z'$, on $q$, and on $Z$, we obtain that there exists two constants $C,C'$ such that
\[\sum_{Z\subset A^*,~Z\neq \emptyset}\sum_{q\geq 1}\sum_{Z'\subset Z,~Z'\neq \emptyset}\sum_{\bk\in [q..\infty)^{Z'}} \eqref{rel:stodom}\leq \pa{C\eps^{-2}\beta\max_i D_i}\pa{\frac{C'\beta}{\sqrt{\beta-1}}\eps^{-4}\beta^2\max_i D_i^2}^{m^*}.\]
Thus, 
\[\sum_{t\in [T]}\EE{\Delta\pa{A_t}\II{\fM_t\wedge\fZ_t\wedge\neg\fC_t}}\leq \Delta_{\max}\pa{C\eps^{-2}\beta\max_i D_i}\pa{\frac{C'\beta}{\sqrt{\beta-1}}\eps^{-4}\beta^2\max_i D_i^2}^{m^*}.\]

\section{Proof of Theorem~\ref{thm:tsclipgauss} (\textsc{clip cts-gaussian} for linear rewards)}
\label{app:optimistic_cts}
In this section, we provide an analysis for the regret bound of \textsc{clip cts-gaussian}, which is stated completely as follows.
 \begin{theorem*}
  The policy \textsc{clip cts-gaussian} has regret bounded by
\begin{align*}
 &\sum_{i\in [n]}\frac{{ 128\pa{ 4\log_2^2(4\sqrt{m}) \beta D_i\log\pa{2^m\abs{\cA}T}\wedge m\Gamma_{ii}\pa{\log(T)+4\log\log(T)}}}}{\Delta_{i,\min}} +
\Delta_{\max}(1+5.2n)\\+&\frac{nm^2\Delta_{\max}}{\pa{\frac{\Delta_{\min}}{2B}-({m^*}(m^*+1)/2+1)\eps}^{2}}+\Delta_{\max}\pa{C\eps^{-2}\beta\max_i D_i}\pa{\frac{C'}{\sqrt{\beta-1}}\eps^{-4}\beta^3\max_i D_i^2}^{m^*},
\end{align*}
where $C,C'$ are two universal constants, and $\eps\in (0,1)$ is such that ${\Delta_{\min}}/(2B)-({m^*}^2+1)\eps>~0.$
 \end{theorem*}
More precisely, notice that the modification on the sample $\btheta_t$ has an impact only in two places in the analysis: in the concentration bound and in the event controlling optimism. We detail these two points in the following.

\subsection{Concentration bound}
In this subsection, we provide the concentration bound of \textsc{clip cts-gaussian}. Our strategy here is to either use the concentration from $\bmu_t$ or from $\btheta_t$, depending on which regime is the best for each arm.
Thus, we define $S\triangleq
\set{i\in [n],~\Gamma_{ii}m\pa{\log(T)+4\log\log(T)}\geq 4\log^2_2(4\sqrt{m})\beta D_i\log\pa{\abs{A}2^mT}}$. We have the following lemma.
\begin{lemma}
\begin{align*} \PPc{{\be_{A_t\cap S}\transpose\pa{\bar\bmu_{t-1}\vee\btheta_{t}\wedge \bmu_t-\bar\bmu_{t-1}} }\geq\sqrt{{2}{\log\pa{\abs{\cA}2^mT}}\sum_{i\in {A_t\cap S}}{\beta D_i}/{N_{i,t-1}}}}{\cH_t}\leq 1/T.\end{align*}
\end{lemma}
\begin{proof}
We define the functions \[\alpha_t(A)\triangleq\sqrt{2{\log\pa{\abs{\cA}2^mT}}\sum_{i\in {A}}\frac{\beta D_i}{N_{i,t-1}}},\quad\text{and}\quad \lambda_t(A) \triangleq \frac{\alpha_t(A)}{\sum_{i\in A}\beta D_i/N_{i,t-1}}, \] we have 
\begin{align*}&\PPc{\be_{A_t\cap S}\transpose\pa{\bar\bmu_{t-1}\vee\btheta_{t}\wedge \bmu_t-\bar\bmu_{t-1}} \geq \alpha_t(A_t\cap S)}{\cH_t}\\&\leq\sum_{A\in \cA}\PPc{\be_{A\cap S}\transpose\pa{\bar\bmu_{t-1}\vee\btheta_{t}-\bar\bmu_{t-1}} \geq \alpha_t(A\cap S)}{\cH_t}\\&\leq \sum_{A\in \cA} e^{-\lambda_t(A\cap S) \alpha_t(A\cap S)}\EEc{e^{\lambda_t(A\cap S) \norm{\be_{A\cap S}\odot\pa{0\vee\pa{\btheta_{t}-\bar\bmu_{t-1}}} }_1}}{\cH_t}
\\&\leq 
\sum_{A\in \cA} e^{-\lambda_t(A\cap S) \alpha_t(A\cap S)}\prod_{i\in A\cap S}\EEc{e^{\lambda_t(A\cap S) \pa{0\vee\pa{\theta_{i,t}-\bar\mu_{i,t-1}} }}}{\cH_t}
\\&\leq \sum_{A\in \cA} e^{-\lambda_t(A\cap S) \alpha_t(A\cap S)}\prod_{i\in A\cap S}\EEc{1+e^{\lambda_t(A\cap S) \pa{{\theta_{i,t}-\bar\mu_{i,t-1}} }}}{\cH_t}
\\&\leq \sum_{A\in \cA} e^{-\lambda_t(A\cap S) \alpha_t(A\cap S)}\prod_{i\in A\cap S}\EEc{2e^{\lambda_t(A\cap S) \pa{{\theta_{i,t}-\bar\mu_{i,t-1}} }}}{\cH_t}
\\&\leq \sum_{A\in \cA} 2^{\abs{A\cap S}}e^{-\lambda_t(A\cap S) \alpha_t(A\cap S)}{e^{\lambda_t(A\cap S)^2 {\sum_{i\in A\cap S}\beta D_i/(2 N_{i,t-1}) }}}\\&\leq 1/T.
\end{align*}
\end{proof}
We now use the definition of $\bmu_t$ to have
\[\be_{A_t\cap S^c}\transpose\pa{\bar\bmu_{t-1}\vee\btheta_{t}\wedge \bmu_t-\bar\bmu_{t-1}}\leq\be_{A_t\cap S^c}\transpose\pa{ \bmu_t-\bar\bmu_{t-1}} = \sum_{i\in A_t\cap S^c}\sqrt{\Gamma_{ii}\frac{2\pa{\log(t)+4\log\log(t)}}{N_{i,t-1}}}. \]
To conclude, we have the following event
\[\fA_t\triangleq \set{\Delta_t\leq \sqrt{{8}{\log\pa{\abs{\cA}2^mT}}\sum_{i\in {A_t\cap S}}{\beta D_i}/{N_{i,t-1}}} + \sum_{i\in A_t\cap S^c}\sqrt{\Gamma_{ii}\frac{8\pa{\log(t)+4\log\log(t)}}{N_{i,t-1}}}}.\]
Using Proposition~\ref{prop:composed}, we have
\begin{align*}
\sum_{t\in [T]}\EE{\Delta_t\II{\fA_t}}&\leq
\sum_{t\in [T]}\EE{\Delta_t\II{\Delta_t\leq 2\sqrt{{8}{\log\pa{\abs{\cA}2^mT}}\sum_{i\in {A_t\cap S}}{\beta D_i}/{N_{i,t-1}}}}}\\&+\sum_{t\in [T]}\EE{\Delta_t\II{\Delta_t\leq 2\sum_{i\in A_t\cap S^c}\sqrt{\Gamma_{ii}\frac{8\pa{\log(t)+4\log\log(t)}}{N_{i,t-1}}}}}
.\end{align*}
We can thus apply Theorem~\ref{thm:l1error} and Theorem~\ref{thm:dege} (see Appendix~\ref{app:gen}) to get the bound
\[512\log_2^2(4\sqrt{m})\sum_{i\in S}\Delta_{i,\min}^{-1}\beta D_i{\log\pa{\abs{\cA}2^mT}}+128m\sum_{i\in S^c} \Delta_{i,\min}^{-1}\Gamma_{ii}\pa{\log(T)+4\log\log(T)}\]

\subsection{Optimism}
In this subsection, we examine the theoretical impact of considering \textsc{clip cts-gaussian} on the optimism-controlling event (event $\neg \fC_t$), in the case of linear rewards. For this purpose, we modify the beginning of Step 4 in the analysis by considering the following events. 
\begin{itemize}
    \item $\fZ_t\triangleq\set{\Delta_t>0}$
    \item $\fC_t\triangleq \set{{{\be_{A_t}\transpose\tilde\btheta_t}}> \be_{A^*}\transpose\bmu^*-\pa{m^*\pa{m^*+1}/2+ 1}\eps}$
    \item $\fR(\btheta',Z)\triangleq \set{\forall A\in \argmax_{A'\in \cA}\be_{A'}\transpose\pa{\btheta'}\text{ we have }Z\subset  A,~{\be_{\ora\pa{\btheta'}}\transpose{\btheta'}}>\be_{A^*}\transpose \bmu^*-\pa{m^*\pa{m^*+1}/2+ 1}\eps}$
    \item $\fS_t\pa{Z}\triangleq \set{\forall \btheta' \text{ s.t. } 0 \leq \pa{\bmu^*-\btheta'}\odot\be_{Z}\leq \eps \be_{Z},~ \fR(\btheta'\odot\be_{Z}+\tilde\btheta_t\odot\be_{Z^c},Z) \text{ holds} }$
    \item $\fT_t\pa{Z}\triangleq \set{\exists i\in Z,~ {{\mu^*_i-\mu^*_i\wedge\tilde\theta_{i,t}}}> \eps}.$
    \item $\fJ_t\triangleq \set{\forall i\in [n], \mu^*_i\leq \mu_{i,t}}$
\end{itemize}
In the above events, $\tilde\btheta_t$ is  $\bmu_t\wedge\btheta_t\vee \bar\bmu_t$. 
The last event $\fJ_t$ holds with probability at least $1-n/(t\log^2(t))$ from Hoeffding's inequality \citep{hoeffding1963probability}. We thus assume that this event hods in the following, since the regret under the complementary event is bounded by $3.2 n\Delta_{\max}$.
We first state the following lemma.
\begin{lemma}
 $$\fZ_t,\neg \fC_t\imp \exists Z\subset A^*,~Z\neq \emptyset~\text{s.t. the event }\fS_t\pa{Z}\wedge\fT_t\pa{Z}\text{ holds.}$$
 \label{lem:fSfT+}
\end{lemma}
This allows us to consider the success probability
\(\PPc{\neg\fT_{t}\pa{Z}}{\cH_{t}}\) in the analysis.
Notice however that $Z\subset \ora\pa{{\pa{\bmu^*\wedge \tilde\btheta_t}\odot\be_{Z}+\tilde\btheta_t\odot\be_{Z^c}}}$, that is guaranteed when $\fS_t\pa{Z}\wedge\neg\fT_t\pa{Z}$ holds, does not necessarily implies that $Z\subset \ora\pa{\tilde\btheta_t}$. 
However, it turns out that
we have $Z\subset A$ for all $A \in\argmax_{A'\in \cA}\be_{A'}\transpose\pa{\pa{\bmu^*\wedge \tilde\btheta_t}\odot\be_{Z}+\tilde\btheta_t\odot\be_{Z^c}}$ implies that $ Z\subset A$ for all $A \in\argmax_{A'\in \cA}\be_{A'}\transpose\pa{ \tilde\btheta_t}$. This last fact is from Lemma~\ref{lem:increa},
with $\bheta=\pa{\bmu^*\wedge \tilde\btheta_t}\odot\be_{Z}+\tilde\btheta_t\odot\be_{Z^c}$ and $\bdelta=\pa{\tilde\btheta_t-\bmu^*\wedge \tilde\btheta_t}\odot\be_{Z}$.
\begin{lemma}\label{lem:increa}
Let $\bheta\in \R^n$, $\bdelta\in \R^n_+$ such that for all $A\in\argmax_{A'\in \cA}\be_{A'}\transpose {\bheta }$, we have $Z\subset A$. Then, for all  $A\in \argmax_{A'\in \cA}\be_{A'}\transpose\pa{\bheta + \bdelta\odot\be_{Z} }$, we have $Z\subset A$.
\end{lemma}
It now remains to explain how to handle the probability \(\PPc{\neg\fT_{t}\pa{Z}}{\cH_{t}}\) in the analysis. Notice that from the high probability event $\fJ_t$, it suffices to treat the case $\tilde\btheta_t=\btheta_t\vee \bar\bmu_t$.
We provide here the places where the analysis differs, the rest of the proof remains unchanged.
\begin{itemize}
    \item We use that \(\PPc{\neg\fT_{t}\pa{Z}}{\cH_{t}}=\PPc{\forall i\in Z,~\eps\vee\pa{\mu^*_i-\bar\mu_{i,t-1}}-0\vee\pa{\theta_{i,t}-\bar\mu_{i,t-1}}\leq \eps}{\cH_{t}},\)
    is a product of functions that are decreasing with respect to $\eps\vee\pa{\mu^*_i-\bar\mu_{i,t-1}}$.
    \item We use that $\eps\vee\pa{\mu^*_i-\bar\mu_{i,t-1}} \geq g_i^{-1}\pa{u_i\vee g_i(\eps)}$ is equivalent to  ${\mu^*_i-\bar\mu_{i,t-1}} \geq g_i^{-1}\pa{u_i\vee g_i(\eps)}$. Thus, we
    don't sum on $\bs$, and can use Assumption~\ref{ass:subgau} with $\blambda\in \R_+^n$. 
\end{itemize}

\begin{proof}[Proof of Lemma~\ref{lem:fSfT+}]
It is sufficient to prove that
 \begin{align}\fZ_t,\neg \fC_t\imp \exists Z\subset A^*,~Z\neq \emptyset~\text{s.t. }\fS_t\pa{Z}\text{ holds,}\label{rel:toprovets}\end{align}
 because $\neg \fC_t$ and $\fS_t\pa{Z}$ together imply  $\fT_t\pa{Z}$. Indeed, see that from $\neg\fT_t\pa{Z}$, we can plug $\btheta'=\bmu^*\wedge \tilde\btheta_t$ into $\fS_t\pa{Z}$ to get
\begin{align*}\be_{A_t}\transpose \tilde\btheta_t&=\max_{A\in \cA} \be_{A}\transpose{\tilde\btheta_t}
\\&
\geq{\max_{A\in \cA} \be_{A}\transpose\pa{\btheta'\odot\be_{Z}+\tilde\btheta_t\odot\be_{Z^c}}}
\\&=
\be_{\ora\pa{{\btheta'\odot\be_{Z}+\tilde\btheta_t\odot\be_{Z^c}}}}\transpose\pa{\btheta'\odot\be_{Z}+\tilde\btheta_t\odot\be_{Z^c}}
\\&>\be_{A^*}\transpose \bmu^*-\pa{m^*\pa{m^*+1}/2+ 1}\eps,\end{align*} giving $\fC_t$.
To prove \eqref{rel:toprovets},
we first consider the choice $Z=Z_1=A^*$. 
Two cases can be distinguished:
\begin{itemize}
    \item[1a)] $\forall \btheta'$ s.t. $0 \leq \pa{\bmu^*-\btheta'}\odot\be_{A^*}\leq \eps \be_{A^*}$, we have $A^*\subset A$ for any action $A \in\argmax_{A'\in \cA}\be_{A'}\transpose\pa{\btheta'\odot\be_{A^*}+\tilde\btheta_t\odot\be_{{A^*}^c}}$.
        \item[1b)] $\exists \btheta'$ s.t. $0 \leq \pa{\bmu^*-\btheta'}\odot\be_{A^*}\leq \eps \be_{A^*}$ such that $A^*\not\subset A$ for some action $A \in\argmax_{A'\in \cA}\be_{A'}\transpose\pa{\btheta'\odot\be_{A^*}+\tilde\btheta_t\odot\be_{{A^*}^c}}$.
\end{itemize}
\textbf{1a)} For the first case, consider any vector $\btheta'$ such that $0 \leq \pa{\bmu^*-\btheta'}\odot\be_{A^*}\numrel{\leq}{finfirstlemm1} \eps \be_{A^*}$ and let $A\numrel{=}{rel0lem1}\ora\pa{\btheta'\odot\be_{A^*}+\tilde\btheta_t\odot\be_{{A^*}^c}}$. We can write
$$\be_{A}\transpose\pa{\btheta'\odot\be_{A^*}+\tilde\btheta_t\odot\be_{{A^*}^c}} \numrel{\geq}{rel1lem1} \be_{A^*}\transpose\pa{\btheta'\odot\be_{A^*}+\tilde\btheta_t\odot\be_{{A^*}^c}} \numrel{\geq}{rel2lem1} \be_{A^*}\transpose\bmu^* - m^*\eps,$$
where \eqref{rel1lem1} is from \eqref{rel0lem1}, and \eqref{rel2lem1} is from \eqref{finfirstlemm1}.
This rewrites as
$$\be_{A}\transpose\pa{\btheta'\odot\be_{A^*}+\tilde\btheta_t\odot\be_{{A^*}^c}} \geq \be_{A^*}\transpose\bmu^* - m^*\eps >  \be_{A^*}\transpose\bmu^* - \pa{m^*\pa{m^*+1}/2+ 1}\eps,$$
so $\fR_t(\btheta'\odot\be_{A^*}+\tilde\btheta_t\odot\be_{{A^*}^c},A^*)$ holds. Therefore, we have proved that $\fS_t\pa{A^*}$ holds.

\textbf{1b)} For the second case, we have some vector $\btheta'$ such that $0 \numrel{\leq}{lastlemma01} \pa{\bmu^*-\btheta'}\odot\be_{A^*}\numrel{\leq}{lastlemma1} \eps  \be_{A^*}$, and some action $A\in\argmax_{A'\in \cA}\be_{A'}\transpose\pa{\btheta'\odot\be_{A^*}+\tilde\btheta_t\odot\be_{{A^*}^c}}$  such that $A^*\not\subset A$.
We consider $Z_2=A^*\cap A $. We first prove that $Z_2\neq \emptyset$ by showing that if an action $S'$ is such that $S'\cap A^*\numrel{=}{relS'lem1}\emptyset$, then $A\neq S'$:
\begin{align*}
    \be_{S'}\transpose \pa{\btheta'\odot\be_{A^*}+\tilde\btheta_t\odot\be_{{A^*}^c}} \numrel{=}{relS'lem1bis} \be_{S'}\transpose \tilde\btheta_t 
    &\numrel{\leq}{relStlem1}
    \be_{A_t}\transpose \tilde\btheta_t 
    \\&\numrel{\leq}{relCtlem1}
    \be_{A^*}\transpose \bmu^*-\pa{m^*\pa{m^*+1}/2+ 1}\eps
    \\&< \be_{A^*}\transpose \bmu^*-m^*\eps
    \\&\numrel{\leq}{relepsslem1} \be_{A^*}\transpose \pa{\btheta'\odot\be_{A^*}+\tilde\btheta_t\odot\be_{{A^*}^c}},
\end{align*}
where \eqref{relS'lem1bis} is from \eqref{relS'lem1}, \eqref{relStlem1} is from the definition of $A_t$, \eqref{relCtlem1} is from $\neg \fC_t$ and \eqref{relepsslem1} is from \eqref{lastlemma1}.
Now, we again distinguish two cases:
\begin{itemize}
    \item[2a)] $\forall \btheta''$ s.t. $0 \leq \pa{\bmu^*-\btheta''}\odot\be_{Z_2}\leq \eps\be_{Z_2}$, we have $Z_2\subset B$ for any action $B \in\argmax_{A'\in \cA}\be_{A'}\transpose\pa{\btheta''\odot\be_{Z_2}+\tilde\btheta_t\odot\be_{{Z_2}^c}}$.
        \item[2b)] $\exists \btheta''$ s.t. $0 \leq \pa{\bmu^*-\btheta''}\odot\be_{Z_2}\leq \eps\be_{Z_2}$ such that  $Z_2\not\subset B$ for some action  $B \in\argmax_{A'\in \cA}\be_{A'}\transpose\pa{\btheta''\odot\be_{Z_2}+\tilde\btheta_t\odot\be_{{Z_2}^c}}$.
\end{itemize}
Notice that when $0 \leq \pa{\bmu^*-\btheta''}\odot\be_{Z_2}\numrel{\leq}{rellem1eps}  \eps\be_{Z_2}$, then
\begin{align}
    \be_{A}\transpose \pa{\btheta''\odot\be_{{Z_2}}+\tilde\btheta_t\odot\be_{{{Z_2}}^c}}&\geq \be_{A}\transpose \pa{\btheta'\odot\be_{{A^*}}+\tilde\btheta_t\odot\be_{{{A^*}}^c}} - (m^*-1)\eps. \label{rel2epslemm1}
\end{align}
Indeed, \eqref{rel2epslemm1} is a consequence of 
\begin{align*}\be_A\transpose\pa{\btheta''\odot\be_{{Z_2}}+\tilde\btheta_t\odot\be_{{{Z_2}}^c}-\btheta'\odot\be_{A^*}-\tilde\btheta_t\odot\be_{{A^*}^c}}&=
\be_{Z_2}\transpose
\pa{\btheta'' - \btheta'}
\\&= {\be_{Z_2}\transpose\pa{\btheta''-\bmu^*}}+{\be_{Z_2}\transpose\pa{\bmu^*-\btheta'}}
\\&\geq -\eps (m^*-1) + 0,
\end{align*}
where we used \eqref{rellem1eps}, \eqref{lastlemma01} and that $Z_2$ is strictly included in $A^*$.

\textbf{2a)} For the first case, considering any vector $\btheta''$ such that $0 \leq \pa{\bmu^*-\btheta''}\odot\be_{Z_2}\leq  \eps\be_{Z_2}$, we have with $B=\ora\pa{\btheta''\odot\be_{Z_2}+\tilde\btheta_t\odot\be_{{Z_2}^c}}$ that
\begin{align*}\be_B\transpose \pa{\btheta''\odot\be_{{Z_2}}+\tilde\btheta_t\odot\be_{{{Z_2}}^c}}&\geq \be_{A}\transpose \pa{\btheta''\odot\be_{{Z_2}}+\tilde\btheta_t\odot\be_{{{Z_2}}^c}}\\&\numrel{\geq}{releps21lem1}
\be_{A}\transpose \pa{\btheta'\odot\be_{{A^*}}+\tilde\btheta_t\odot\be_{{A^*}^c}} - (m^*-1)\eps
\\&\geq\be_{A^*}\transpose \pa{\btheta'\odot\be_{{A^*}}+\tilde\btheta_t\odot\be_{{A^*}^c}} - (m^*-1)\eps\\&\numrel{\geq}{otherlem1} 
\be_{A^*}\transpose \bmu^* -m^*\eps -(m^*-1)\eps ,
\end{align*}
where \eqref{releps21lem1} uses \eqref{rel2epslemm1} and \eqref{otherlem1} uses \eqref{lastlemma1}.
This rewrites as
$$\be_B\transpose \pa{\btheta''\odot\be_{{Z_2}}+\tilde\btheta_t\odot\be_{{{Z_2}}^c}} \geq  \be_{A^*}\transpose \bmu^*-\pa{m^*\pa{m^*+1}/2+ 1}\eps, $$
so $\fR_t(\btheta'\odot\be_{{Z_2}}+\tilde\btheta_t\odot\be_{{{Z_2}}^c},{Z_2})$ holds, and thus we proved that $\fS_t(Z_2)$ holds. 

\textbf{2b)} For the second case, we have a vector $\btheta''$ such that $0 \leq \pa{\bmu^*-\btheta''}\odot\be_{Z_2}\leq  \eps\be_{Z_2}$ and an action $B\in \argmax_{A'\in \cA}\be_{A'}\transpose\pa{\btheta''\odot\be_{Z_2}+\tilde\btheta_t\odot\be_{{Z_2}^c}}$ such that $Z_2\not\subset B$. We consider ${Z_3}=Z_2\cap B$. Again, $Z_3\neq \emptyset$ because for any $S'$ such that $S'\cap Z_2 = \emptyset$, we have $S'\neq \ora\pa{\btheta''\odot\be_{Z_2}+\tilde\btheta_t\odot\be_{{Z_2}^c}}$:
\begin{align*}
    \be_{S'}\transpose \pa{\btheta''\odot\be_{Z_2}+\tilde\btheta_t\odot\be_{{Z_2}^c}} = \be_{S'}\transpose \tilde\btheta_t 
    &\leq
    \be_{A_t}\transpose \tilde\btheta_t 
    \\&\leq
    \be_{A^*}\transpose \bmu^*-\pa{m^*\pa{m^*+1}/2+ 1}\eps
    \\&< \be_{A^*}\transpose \bmu^*-(m^*+(m^*-1))\eps
    \\&\leq \be_{A}\transpose \pa{\btheta''\odot\be_{Z_2}+\tilde\btheta_t\odot\be_{{Z_2}^c}},
\end{align*}
where the last inequality is obtained in the same way as in inequalities from \eqref{releps21lem1} to \eqref{otherlem1}.

We could repeat the above argument and each time the size $Z_i$ is decreased by at least $1$. Thus, after at most $m^*-1$ steps, since $m^*+ (m^*-1) + (m^*-2)+\dots+ 1 =m^*\pa{m^*+1}/2$ is still less than $m^*\pa{m^*+1}^2/2+ 1$, we could reach the end and find a $Z_i\neq \emptyset$ such that $\fS_t\pa{Z_i}$ holds. 
\end{proof}

\begin{proof}[Proof of Lemma~\ref{lem:increa}]
Let's prove that $\argmax_{A'\in \cA}\be_{A'}\transpose \pa{\bheta + \bdelta\odot\be_{Z} }\subset \argmax_{A'\in \cA}\be_{A'}\transpose \bheta$.
Consider 
any action $A\in \argmax_{A'\in \cA}\be_{A'}\transpose \pa{\bheta + \bdelta\odot\be_{Z} }$.
If $A\notin \argmax_{A'\in \cA}\be_{A'}\transpose \bheta$, then there exists $B\in \argmax_{A'\in \cA}\be_{A'}\transpose \bheta$ such that
\[\be_{A}\transpose \bheta < \be_{B}\transpose \bheta .\]
Furthermore, since $Z\subset B$ and $\bdelta\geq 0$, we also have
\[\be_{A}\transpose\pa{\bdelta\odot\be_{Z}} \leq \be_{B}\transpose\pa{\bdelta\odot\be_{Z}},\]
so we finally have
\[\be_{A}\transpose\pa{\bheta+\bdelta\odot\be_{Z}}< \be_{B}\transpose\pa{\bheta+\bdelta\odot\be_{Z}},\]
contradicting that $A\in \argmax_{A'\in \cA}\be_{A'}\transpose \pa{\bheta + \bdelta\odot\be_{Z} }$.
\end{proof}

\section{General CMAB results}\label{app:gen}

In this section, we state general results that are useful for every regret analysis that we conducted in this paper. The main result of the section is the following theorem, inspired from the analysis of \citet{Degenne2016}, that gives a regret bound under the event that the gap $\Delta_t$ is controlled by a $\ell_2$ norm type error.

\begin{theorem}[Regret bound for $\ell_2$-norm error]\label{thm:dege}For all $i\in [n]$, let ${\beta_{i,T}}\in  \R_+$.
For $t\geq 1$,
consider the event
\[\fA_t\triangleq\set{\Delta_t\leq
\norm{\sum_{i\in A_t}\frac{\beta_{i,T}^{1/2}\be_i}{N_{i,t-1}^{1/2}}}_2}.\]
Then, \[\sum_{t=1}^T\II{\fA_t}\Delta_t\leq 32\log_2^2(4\sqrt{m})\sum_{i\in [n]}{\beta_{i,T}\Delta_{i,\min}^{-1}}.\]
\end{theorem}

\begin{proof} Let $t\geq 1$. We define $\Lambda_t\triangleq \norm{\sum_{i\in A_t}{\beta_{i,T}^{1/2}}{N_{i,t-1}^{-1/2}\be_i}}_2$. 
We start by a simple lower bound on $\Lambda_t$, holding for any $j\in A_t$, \begin{align}\Lambda_t\geq \norm{\frac{\beta_{j,T}^{1/2}\be_j}{N_{j,t}^{1/2}}}_2=\frac{\beta_{j,T}^{1/2}}{N_{j,t}^{1/2}}.\label{LowerLambda_t}\end{align}
We then use the same reverse amortisation technique than in \citet{wang2017improving}.  
\begin{align*}
    \Lambda_t&=-\Lambda_t+\norm{\sum_{i\in A_t}\frac{2\beta_{i,T}^{1/2}\be_i}{N_{i,t-1}^{1/2}}}_2
    \\
    &=
    -\norm{\frac{\Lambda_t\be_{A_t}}{\norm{\be_{A_t}}_2}}_2+\norm{\sum_{i\in A_t}\frac{2\beta_{i,T}^{1/2}\be_i}{N_{i,t-1}^{1/2}}}_2
        \\
    &\leq 
    \norm{\sum_{i\in A_t}\pa{\frac{2\beta_{i,T}^{1/2}}{N_{i,t-1}^{1/2}}-\frac{\Lambda_t}{\norm{\be_{A_t}}_2}}^+\be_i }_2 
          \\
              &= 
    \norm{\sum_{i\in A_t}\pa{\frac{2\beta_{i,T}^{1/2}}{N_{i,t-1}^{1/2}}-\frac{\Lambda_t}{\norm{\be_{A_t}}_2}}^+\II{\Lambda_t\geq \frac{\beta_{i,T}^{1/2}}{N_{i,t-1}^{1/2}} }\be_i }_2 &\text{Using } \eqref{LowerLambda_t}
          \\
    &\leq \norm{\sum_{i\in A_t}\II{2\Lambda_t\geq\frac{2\beta_{i,T}^{1/2}}{N_{i,t-1}^{1/2}}\geq \frac{\Lambda_t}{\norm{\be_{A_t}}_2}}\frac{2\beta_{i,T}^{1/2}\be_i}{N_{i,t-1}^{1/2}} }_2.
\end{align*}
We now decompose the interval $[2,{1}/{\norm{\be_{A_t}}_2}]$ using a peeling:
\[[2,{1}/{\norm{\be_{A_t}}_2}]\subset\bigcup_{k=0}^{\ceil{\log_2\pa{\norm{\be_{A_t}}_2}}} [2^{-k},2^{1-k}].\]
This induces a partition of the set of indices:
\[\II{i\in A_t,~2\Lambda_t\geq\frac{2\beta_{i,T}^{1/2}}{N_{i,t-1}^{1/2}}\geq \frac{\Lambda_t}{\norm{\be_{A_t}}_2}}\subset \bigcup_{k=0}^{\ceil{\log_2\pa{\norm{\be_{A_t}}_2}}} J_{k,t}, \]
where for all interger $1\leq k\leq {\ceil{\log_2\pa{\norm{\be_{A_t}}_2}}} $,
\[J_{k,t}\triangleq \set{i\in A_t,~2^{1-k}\Lambda_t\geq\frac{2\beta_{i,T}^{1/2}}{N_{i,t-1}^{1/2}}\geq {2^{-k}\Lambda_t}}.\]
We can thus
upper bound $\Lambda_t^2$ using this decomposition
\begin{align*}\Lambda_t^2&\leq \norm{\sum_{i\in A_t}\II{2\Lambda_t\geq\frac{2\beta_{i,T}^{1/2}}{N_{i,t-1}^{1/2}}\geq \frac{\Lambda_t}{\norm{\be_{A_t}}_2}}\frac{2\beta_{i,T}^{1/2}\be_i}{N_{i,t-1}^{1/2}} }_2^2
\\&\leq \sum_{k=0}^{\ceil{\log_2\pa{\norm{\be_{A_t}}_2}}}\norm{\sum_{i\in J_{k,t}}{\frac{2\beta_{i,T}^{1/2}\be_i}{N_{i,t-1}^{1/2}}}}^2_2
\\&\leq 
\sum_{k=0}^{\ceil{\log_2\pa{\norm{\be_{A_t}}_2}}}2^{2-2k}\Lambda_t^2\norm{{{\be_{J_{k,t}}}} }_2^2.
\end{align*}
This last inequality implies that there must exist one integer $k_t$ such that $\abs{J_{k_t,t}}=\norm{\be_{J_{k_t,t}}}^2_2\geq 2^{2k_t-2}\pa{1+\ceil{\log_2\pa{\norm{\be_{A_t}}_2}}}^{-1}$. We now upper bound ${\sum_{t=1}^T\II{\fA_t}\Delta_t}$, using $\abs{A_t}\leq m$, i.e., \[\ceil{\log_2\pa{\norm{\be_{A_t}}_2}}\leq \ceil{\log_2(m)/2}.\]
\begin{align*}{\sum_{t=1}^T\II{\fA_t}\Delta_t}&\leq{\sum_{t=1}^T\sum_{k=0}^{\ceil{\log_2(m)/2}}\II{k_t=k,~\fA_t}\Delta_t}
\\&\leq
{\sum_{t=1}^T\sum_{k=0}^{\ceil{\log_2(m)/2}}{\II{k_t=k,~\fA_t}\sum_{i\in I}\II{i\in J_{k,t}}}\Delta_t2^{2-2k}\pa{\ceil{\log_2(m)/2}+1}}
\\&\leq
{\sum_{t=1}^T\sum_{k=0}^{\ceil{\log_2(m)/2}}{\sum_{i\in I}\II{i\in A_t,~N_{i,t-1}^{1/2}\leq \frac{2^{k+1}\beta_{i,T}^{1/2}}{\Delta_t}}}\Delta_t2^{2-2k}\pa{\ceil{\log_2(m)/2}+1}}
\\&=
\pa{\ceil{\log_2(m)/2}+1}\sum_{k=0}^{\ceil{\log_2(m)/2}}2^{2-2k}\sum_{i\in I}\underbrace{{\sum_{t=1}^T{\II{i\in A_t,~N_{i,t-1}^{1/2}\leq \frac{2^{k+1}\beta_{i,T}^{1/2}}{\Delta_{t}}}}\Delta_t}}_{\numterm{sumcounter}_{i,k}}.
\end{align*}

\vspace{-.3cm}

\noindent
Applying Proposition~\ref{prop:sumcounter} gives
\[\eqref{sumcounter}_{i,k}\leq \frac{\beta_{i,T}2^{\frac{k+1}{1/2}}}{1-1/2}\Delta_{i,\min}^{1-1/1/2}  
.\]
So we get, using $\ceil{\log_2(m)/2}+1\leq \log_2(4\sqrt{m})$,
\[\sum_{t=1}^T\II{\fA_t}\Delta_t\leq 32\log_2^2(4\sqrt{m})\sum_{i\in [n]}{\beta_{i,T}\Delta_{i,\min}^{-1}}.\]
\end{proof}
The following Proposition~\ref{prop:sumcounter} is a standard and general result in CMAB, that was first proved in \citet{chen13a}. 
\begin{proposition} Let $i\in [n]$ and $f_i: \R_+\to \R_+$ be a non increasing function, integrable on $[\Delta_{i,\min},\Delta_{i,\max}]$. Then
\[\sum_{t=1}^T{\II{i\in A_t,~ N_{i,t-1}\leq f_i(\Delta_t)}}{\Delta_t} \leq
{f_i(\Delta_{i,\min})}\Delta_{i,\min}+
\int_{\Delta_{i,\min}}^{\Delta_{i,\max}}{f_i(x)\emph{d}x}.\]
  \label{prop:sumcounter}
 \end{proposition}
\begin{proof}
Consider $\Delta_{i,\max}=\Delta_{i,1}\geq \Delta_{i,2}\geq \dots \geq\Delta_{i,K_i}=\Delta_{i,\min}$ being all possible values for $\Delta_t$ when $i\in A_t$.
We define a dummy gap $\Delta_{i,0}=\infty$ and let $f_i\pa{\Delta_{i,0}}=0$.
In \eqref{rangebreak}, we first break the range $(0,f_i(\Delta_t)]$ of the counter $N_{i,t-1}$ into sub intervals: \[(0,f_i(\Delta_t)]=(f_i(\Delta_{i,0}),f_i(\Delta_{i,1})]\cup\dots\cup (f_i(\Delta_{i,k_t-1}),f_i(\Delta_{i,k_t})],\] where $k_t$ is the index such that $\Delta_{i,k_t} = \Delta_t$. {This index $k_t$ exists by assumption that the subdivision contains all possible values for $\Delta_t$ when $i\in A_t$. Notice that in \eqref{rangebreak}, we do not explicitly use $k_t$, but instead sum over all $k\in [K_i]$ and filter against the event $\set{\Delta_{i,k}\geq \Delta_t}$, which is equivalent to summing over $k\in [k_t].$}  

\begin{align}
&\sum_{t=1}^T{\II{i\in A_t,~ N_{i,t-1}\leq f_i(\Delta_t)}}{\Delta_t}\nonumber
\\&=\label{rangebreak}
\sum_{t=1}^T\sum_{k=1}^{K_i}\II{i\in A_t,~f_i(\Delta_{i,k-1})< N_{i,t-1}\leq f_i(\Delta_{i,k}),\Delta_{i,k}\geq \Delta_t}{\Delta_t}.
\end{align}
Over each event that $N_{i,t-1}$ belongs to the interval $(f_i(\Delta_{i,k-1}),f_i(\Delta_{i,k})]$, we upper bound the suffered gap $\Delta_t$ by $\Delta_{i,k}$. 

\begin{align}
\eqref{rangebreak}&\leq\label{uppergap}
\sum_{t=1}^T\sum_{k=1}^{K_i}\II{i\in A_t,~f_i(\Delta_{i,k-1})< N_{i,t-1}\leq f_i(\Delta_{i,k}),\Delta_{i,k}\geq \Delta_t}{\Delta_{i,k}}.
\end{align}
Then, we further upper bound the summation by adding events that $N_{i,t-1}$ belongs to the remaining intervals
$(f_i(\Delta_{i,k-1}),f_i(\Delta_{i,k})]$ for $k_t<k\leq K_i$, associating them to a suffered gap $\Delta_{i,k}$. This is equivalent to removing the filtering against the event $\set{\Delta_{i,k}\geq \Delta_t}$. 

\begin{align}
\eqref{uppergap}&\leq\label{addremainingintervals}
\sum_{t=1}^T\sum_{k=1}^{K_i}\II{i\in A_t,~f_i(\Delta_{i,k-1})< N_{i,t-1}\leq f_i(\Delta_{i,k})}{\Delta_{i,k}}.\end{align}
Now, we invert the summation over $t$ and the one over $k$. 

\begin{align}
\eqref{addremainingintervals}&=\label{sumswitch}
\sum_{k=1}^{K_i}\sum_{t=1}^T\II{i\in A_t,~f_i(\Delta_{i,k-1})< N_{i,t-1}\leq f_i(\Delta_{i,k})}{\Delta_{i,k}}.
\end{align}
For each $k\in[K_i]$, the number of times $t\in [T]$ that the counter $N_{i,t-1}$ belongs to $(f_i(\Delta_{i,k-1}),f_i(\Delta_{i,k})]$ can be upper bounded  by the number of integers in this interval. This is due to the event $\set{i\in A_t}$, imposing that $N_{i,t-1}$ is incremented, so $N_{i,t-1}$ cannot be worth the same integer for two different times $t$ satisfying $i\in A_t$. We use the fact that for all $x,y\in \R$, $x\leq y$, the number of integers in the interval $(x,y]$ is exactly $\floor{y}-\floor{x}$. 

\begin{align}
\eqref{sumswitch}&\leq\label{intcount}
\sum_{k=1}^{K_i}\pa{\floor{f_i(\Delta_{i,k})}-\floor{f_i(\Delta_{i,k-1})}}{\Delta_{i,k}}.
\end{align}
We then simply expand the summation, and some terms are cancelled (remember that $f_i\pa{\Delta_{i,0}}=0$).

\begin{align}
\eqref{intcount}&=\label{expand}
\floor{f_i(\Delta_{i,K_i})}\Delta_{i,K_i}+
\sum_{k=1}^{K_i-1}\floor{f_i(\Delta_{i,k})}\pa{\Delta_{i,k}-\Delta_{i,k+1}}
\end{align}
We use $\floor{x}\leq x$ for all $x\in \R$. Finally, we recognize a right Riemann sum, and use the fact that $f_i$ is non increasing
to upper bound each $f_i(\Delta_{i,k})\pa{\Delta_{i,k}-\Delta_{i,k+1}}$ by $\int_{\Delta_{i,k+1}}^{\Delta_{i,k}}f_i(x)\text{d}x$, for all $k\in [K_i-1]$.
\begin{align}
\eqref{expand}&\leq\label{upperfloor}
{f_i(\Delta_{i,K_i})}\Delta_{i,K_i}+
\sum_{k=1}^{K_i-1}{f_i(\Delta_{i,k})}\pa{\Delta_{i,k}-\Delta_{i,k+1}}
\\&\leq\label{Riemann}
{f_i(\Delta_{i,K_i})}\Delta_{i,K_i}+
\int_{\Delta_{i,K_i}}^{\Delta_{i,1}}{f_i(x)\text{d}x}.
\end{align}
\end{proof}

There also exist a version for the $\ell_1$-norm error.
\begin{theorem}[Regret bound for $\ell_1$-norm error]\label{thm:l1error}For all $i\in [n]$, let ${\beta_{i,T}}\in  \R_+$.
For $t\geq 1$,
consider the event
\[\fA_t\triangleq\set{\Delta_t\leq
\norm{\sum_{i\in A_t}\frac{\beta_{i,T}^{1/2}\be_i}{N_{i,t-1}^{1/2}}}_1}.\]
Then, \[\sum_{t=1}^T\II{\fA_t}\Delta_t\leq \sum_{i\in [n]}{\beta_{i,T}}{{8m\Delta_{i,\min}}^{-1}}.\]
\end{theorem}
\begin{proof}Let $t\geq 1$.
The first step is the reverse amortisation technique, that allows us to modify the upper bound on $\Delta_t$ in such a way that indices $i$ such that $N_{i,t-1}$ is high enough are removed.
 Assuming that $\fA_t$ holds, we get
\begin{align*}
    \Delta_t\leq {\sum_{i\in A_t}\II{\frac{2\beta_{i,T}^{1/2}}{N_{i,t-1}^{1/2}}\geq \frac{\Delta_t}{{m}}}\frac{2\beta_{i,T}^{1/2}}{N_{i,t-1}^{1/2}} }
\end{align*}
Now, we apply Proposition~\ref{prop:sumcounterter}.
In summary, we have that $\sum_{t=1}^T\II{\fA_t}\Delta_t$ is upper bounded by
\[\sum_{i\in [n]}{\beta_{i,T}}{{8m\Delta_{i,\min}}^{-1}}.\]
\end{proof}

\begin{proposition} Let $i\in [n]$ and $f_i(x)=\beta_{i,T}x^{-1/\alpha_i}$, $\alpha_i\in (0,1]$ and $\beta_{i,T}\geq 0$. Then
\begin{align*}\sum_{t=1}^T{\II{i\in A_t,~\delta_t\neq 0,~N_{i,t-1}\leq f_i(\delta_t)}}{f_i^{-1}\pa{N_{i,t-1}}} \leq \delta_{i,\min}^{1-1/\alpha_i}\frac{\beta_{i,T}}{1-\alpha_i}\II{\alpha_i<1}\\+ \II{\alpha_i=1}\beta_{i,T}\pa{1+\log\pa{\frac{\beta_{i,T}}{\delta_{i,\min}}}}
.\end{align*}
  \label{prop:sumcounterter}
 \end{proposition}
\begin{proof}
We upper bound $f_i\pa{\delta_t}$ by $f_i\pa{\delta_{i,\min}}$ directly in the event, and then simply count the number of integers in $(0,f_i\pa{\delta_{i,\min}}]$. For each such integer $s$, the regret suffered is $f_i^{-1}\pa{s}$. We then upper bound the sum by an integral (using the fact that $f_i^{-1}$ is decreasing), to get the final result.
\begin{align*}\sum_{t=1}^T{\II{i\!\in\! A_t,~\delta_t\!\neq\! 0,~N_{i,t-1}\!\leq\! f_i(\delta_t)}}{f_i^{-1}\pa{N_{i,t-1}}} 
&\leq
\sum_{t=1}^T{\II{i\!\in\! A_t,~N_{i,t-1}\!\leq\! f_i(\delta_{i,\min})}}{f_i^{-1}\pa{N_{i,t-1}}} \\&
\leq 
\sum_{s=1}^{\floor{f_i(\delta_{i,\min})}}
f_i^{-1}(s)\\&\leq f_i^{-1}(1) +
\int_1^{f_i(\delta_{i,\min})}f_i^{-1}(s)\text{d} s
\\&=
\beta_{i,T}^{\alpha_i}+
 \int_1^{\beta_{i,T}\delta_{i,\min}^{-1/\alpha_i}}\beta_{i,T}^{\alpha_i} s^{-\alpha_i}\text{d} s\\&\leq
\II{\alpha_i<1}\delta_{i,\min}^{1-1/\alpha_i}\frac{\beta_{i,T}}{1-\alpha_i} \\&+ \II{\alpha_i=1}\beta_{i,T}\pa{1+\log\pa{\frac{\beta_{i,T}}{\delta_{i,\min}}}}
.\end{align*}
\end{proof}

\begin{proposition}[Regret bound for a composed bonus] Let $K\in \N^*$.
For all $t\geq 1$,
consider the event
$$\fA_t\triangleq\set{\Delta_t\leq
\sum_{k\in [K]}B_{k,t}},$$
for some $B_{k,t}\geq 0$.
Then,  the event-filtered regret $\EE{\sum_{t=1}^T \Delta_t\II{ \fA_t}}$ is upper bounded by
$$\sum_{k\in [K]}\EE{\sum_{t\in [T]}\Delta_t\II{\Delta_t\leq KB_{k,t}}}.
$$
\label{prop:composed}
\end{proposition}
\begin{proof}
From $\fA_t$, there must exists one $k$ such that $\Delta_t\leq
KB_{k,t}$. So $1\leq \sum_{k\in [K]} \II{\Delta_t\leq KB_{k,t}}$, i.e., $\Delta_t\leq \sum_{k\in [K]} \Delta_t\II{\Delta_t\leq KB_{k,t}}$.
\end{proof}

\end{document}